\documentclass{article} 
\usepackage{iclr2026_conference,times}
\usepackage[dvipsnames,table]{xcolor}
\usepackage[colorlinks=true, linkcolor=ForestGreen!70!black, citecolor=ForestGreen!70!black, urlcolor=ForestGreen!70!black]{hyperref}

\usepackage{url}
\usepackage{amsmath,amsfonts,amssymb,bm}

\newtheorem{assumption}{Assumption}

\newenvironment{proof}[1][Proof]{\par\noindent\emph{#1.} }{\hfill$\square$\par}
\usepackage{multirow}
\usepackage{mathtools}
\usepackage{wrapfig}
\usepackage{enumitem}
\usepackage{booktabs}
\usepackage{tabularray}
\usepackage{algorithm, algpseudocode}
\usepackage{thmtools,thm-restate}

\usepackage[nameinlink]{cleveref}
\crefname{theorem}{theorem}{theorems}
\Crefname{theorem}{Theorem}{Theorems}
\crefname{proposition}{proposition}{propositions}
\Crefname{proposition}{Proposition}{Propositions}
\crefname{assumption}{assumption}{assumptions}
\Crefname{assumption}{Assumption}{Assumptions}
\crefname{definition}{definition}{definitions}
\Crefname{definition}{Definition}{Definitions}
\crefname{corollary}{corollary}{corollaries}
\Crefname{corollary}{Corollary}{Corollaries}
\crefname{lemma}{lemma}{lemmas}
\Crefname{lemma}{Lemma}{Lemmas}
\crefname{figure}{Fig.}{Figs.}
\Crefname{figure}{Figure}{Figures}
\crefname{appendix}{Appendix}{Appendices}
\Crefname{appendix}{Appendix}{Appendices}

\title{Tree Reward-Aligned Search for TReASURe in Masked Diffusion Language Models}

\author{
Zichao Yu$^{1, *}$\;
Ming Li$^{2, *}$\;
Wenyi Zhang$^{1}$\;
Weiguo Gao$^{2, \dagger}$    \\
$^*$Equal Contribution \quad
$^\dagger$Corresponding author
\\
$^1$University of Science and Technology of China \quad
$^2$Fudan University\\
$^1$\texttt{zichaoyu@mail.ustc.edu.cn, wenyizha@ustc.edu.cn}\\
$^2$\texttt{mingli23@m.fudan.edu.cn, wggao@fudan.edu.cn}
\vspace{-4mm}
}

\iclrfinalcopy
\begin{document}

\maketitle

\begin{abstract}
Tree search has recently emerged as a powerful framework for aligning generative models with task-specific rewards at test time. 
Applying tree search to Masked Diffusion Language Models, however, introduces two key challenges: (\romannumeral1) parallel unmasking yields highly correlated branches, limiting exploration, and (\romannumeral2) reward evaluation via sampled completions produces high-variance estimates, making pruning unstable. 
We propose \textsc{TReASURe}, a tree-search test-time alignment method that addresses these issues. 
It introduces (\romannumeral1) \textsc{UnmaskBranch}, a branching strategy based on first-hitting unmasking that diversifies both token content and reveal order with a single model call per parent node, and (\romannumeral2) \textsc{ResubstituteScore}, a pruning rule that uses deterministic resubstitution to score partially masked sequences with low-variance proxy completions. 
Theoretically, we quantify branching efficiency gains in NFEs (number of function evaluations), show that the scoring rule approximates the true reward with error bounded by predictive uncertainty, and prove improvements with larger tree widths. 
Empirically, \textsc{TReASURe} achieves state-of-the-art results on perplexity, linguistic acceptability, and control of sentiment and toxicity, outperforming prior methods under matched compute budgets, with especially strong gains in low-NFE regimes.
\end{abstract}

\section{Introduction}

\begin{figure}[b]
\centering
\includegraphics[width=0.99\linewidth]{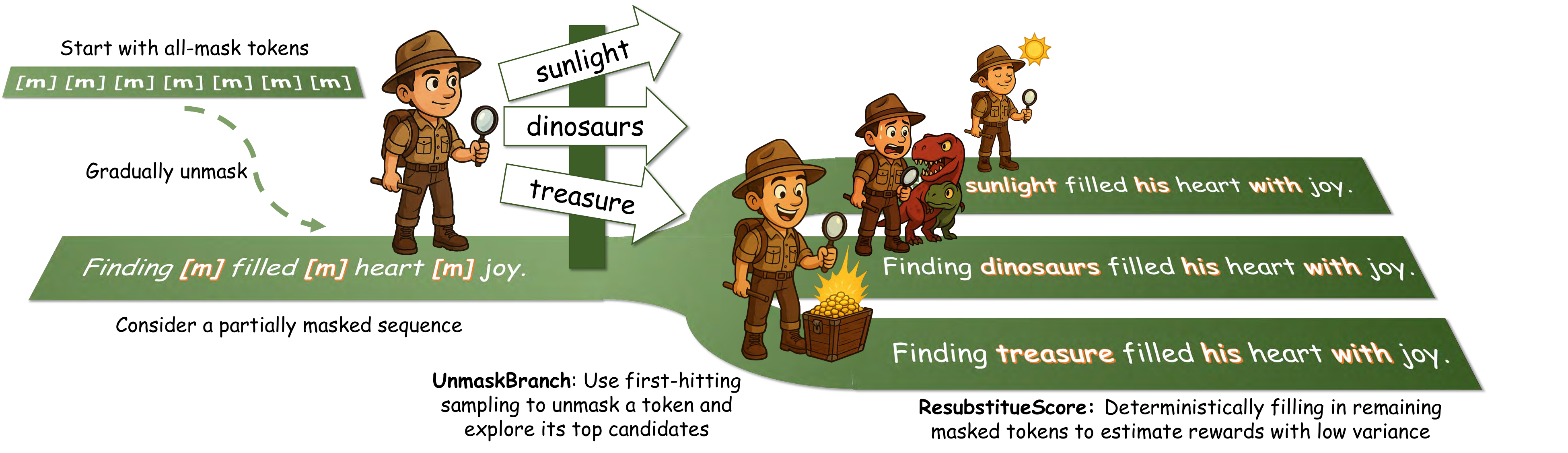}
\caption{\textbf{Conceptual illustration of \textsc{TReASURe}.} \textsc{UnmaskBranch} uses first-hitting sampling to branch by selecting both which position to unmask next and which token to place there, thereby diversifying exploration. \textsc{ResubstituteScore} deterministically fills in the remaining mask tokens to obtain low-variance reward estimates for pruning.}
\label{fig:illustration}
\end{figure}

Masked Diffusion Language Models (MDLMs)~\citep{nie2025large,sahoo2024simple,shi2024simplified,yang2025mmada} have emerged as a compelling alternative to autoregressive models~\citep{brown2020language,radford2019language,touvron2023llama}. They start with all-mask tokens and gradually reveal tokens through a sequence of discrete denoising steps. At each step, the model predicts token distributions for masked positions, conditioned on the current partially masked sequence and the diffusion timestep. This formulation enables flexible sampling schedules and broad conditioning patterns, making MDLMs well-suited for controllable generation tasks.

However, this flexibility is not fully realized without mechanisms to align the model's outputs with user-defined objectives. Test-Time Alignment (TTA) enables guiding language model outputs toward task-specific goals without retraining. In applications such as toxicity avoidance~\citep{logacheva2022paradetox}, sentiment control~\citep{barbieri2020tweeteval}, or enforcing linguistic acceptability~\citep{warstadt2019neural}, aligning generation with external reward functions at test time offers a flexible and training-free alternative to supervised fine-tuning. While TTA has been actively explored in autoregressive~\citep{liu2024aligning,lu2022quark,ziegler2019fine} and continuous diffusion models~\citep{guo2025training,li2025dynamic,singhal2025general,uehara2024understanding,uehara2024bridging}, its application to MDLMs remains limited. To our knowledge, only a few recent works have attempted to integrate reward signals into MDLM decoding at test time. For example, \citet{singhal2025general} propose Feynman--Kac steering, but their approach assumes continuous-state dynamics that may not translate well to discrete, token-level masked diffusion. \citet{pani2025test} introduce a sequential Monte Carlo method, but their evaluation is limited to image generation tasks. See~\Cref{app:additional_literature_review} for a broader discussion.

Tree search has recently shown strong performance in aligning continuous diffusion models~\citep{guo2025training,li2025dynamic} at test time, offering a principled framework for balancing exploration and exploitation. However, applying tree search to MDLMs poses unique challenges. First, branching is ineffective under parallel unmasking: na\"ive updates often yield negligible changes, producing highly correlated trajectories and poor exploration. Second, pruning is unstable in discrete spaces: unlike continuous diffusion, where smooth latent dynamics enable reliable intermediate value estimates, MDLMs output categorical distributions per masked position, making sampled reward estimates high-variance and brittle to small logit perturbations. Overcoming these challenges requires rethinking both branching and pruning.

We propose \textsc{TReASURe} (\textbf{T}ree \textbf{Re}ward-\textbf{A}ligned \textbf{S}earch with \textbf{U}nmasking and \textbf{Re}substitution), a tree-search method designed for MDLMs. \textsc{TReASURe} introduces a branching rule based on \emph{first-hitting unmasking}, which expands the search only at commitment events, preserving efficiency while diversifying unmasking order and token content. For pruning, it employs \emph{resubstitution scoring}, which deterministically fills masked positions to provide low-variance reward estimates with minimal model calls. A conceptual illustration can be found in~\Cref{fig:illustration}. Theoretically, we quantify its efficiency gains in NFEs, show that its scoring rule approximates the true reward with error bounded by the model's predictive uncertainty, and establish provable improvements with larger tree widths. Empirically, across controllable generation tasks (perplexity, linguistic acceptability, toxicity, and sentiment), \textsc{TReASURe} achieves state-of-the-art rewards under matched compute budgets, outperforming na\"ive sampling, Best-of-\(N\), and Feynman--Kac steering.

\paragraph{Contributions.} Our contributions can be summarized as follows: \textbf{(\romannumeral1)} A new perspective on TTA for MDLMs via tree search; \textbf{(\romannumeral2)} a branching rule that exploits unmasking events for efficient, diverse exploration; \textbf{(\romannumeral3)} a pruning rule based on deterministic resubstitution for low-variance reward estimation under fixed NFE; \textbf{(\romannumeral4)} theoretical guarantees including branching efficiency, pruning accuracy, and reward gains; and \textbf{(\romannumeral5)} state-of-the-art performance across extensive controllable generation benchmarks under a fixed compute budget.

\paragraph{Notation.} Let \(\mathcal{V}\) be the set of one-hot vectors in \(\mathbb{R}^V\), with the \(V\)th component reserved for the mask token \(\bm{m}\). Discrete variables are denoted by \(\bm{z}_t, \bm{z}_n, \bm{x} \in \mathcal{V}\), where subscripts \(t\) and \(n\) indicate time and the number of masked tokens. We write \(\bm{x} \sim \mathrm{Cat}(\bm{x};\bm{p})\) if \(\bm{x}\) is drawn from a categorical distribution with parameter \(\bm{p} \in \Delta^V\), the probability simplex. For length-\(L\) sequences, we write \(\bm{z}_t^{1:L}, \bm{z}_n^{1:L}, \bm{x}^{1:L} \in \mathcal{V}^L\), with \(\bm{z}_t^\ell, \bm{z}_n^\ell, \bm{x}^\ell\) denoting the \(\ell\)th token. Finally, \(\mathrm{TopK}_b(\bm{\mu})\) returns the indices of the top \(b\) entries of \(\bm{\mu} \in \Delta^V\).

\section{Background}

\subsection{Masked Diffusion Language Models}

We provide the necessary MDLM background here, with further details in~\Cref{app:more_details_on_mdlms}.

\paragraph{Forward and reverse processes.} MDLMs~\citep{sahoo2024simple,shi2024simplified} define a forward process that mixes data with the absorbing mask token \(\bm{m}=(0,\dotsc,0,1)\in\Delta^V\):
\begin{equation}
q(\bm{z}_t | \bm{x})=\mathrm{Cat}(\bm{z}_t;\alpha_t \bm{x}+(1-\alpha_t)\bm{m}),
\end{equation}
where \(\alpha_t\) decreases monotonically from \(\alpha_0\approx1\) to \(\alpha_1\approx0\). Once a token becomes masked at time \(s\), it remains so for all \(t>s\), i.e., \(q(\bm{z}_t|\bm{z}_s=\bm{m})=\mathrm{Cat}(\bm{z}_t;\bm{m})\). Applied to a sequence, this corruption acts independently across positions, so tokens evolve in parallel. Learning the reverse process thus amounts to iterative unmasking. Conditioned on \(\bm{x}\), the time reversal of the forward process for \(s < t\) is
\begin{equation}
\label{eq:time_reversal}
q(\bm{z}_s | \bm{z}_t,\bm{x})=
\begin{cases}
\mathrm{Cat}(\bm{z}_s;\bm{z}_t), & \bm{z}_t\neq \bm{m},\\
\mathrm{Cat}\Bigl(\bm{z}_s;\dfrac{(1-\alpha_s)\bm{m}+(\alpha_s-\alpha_t)\bm{x}}{1-\alpha_t}\Bigr), & \bm{z}_t=\bm{m}.
\end{cases}
\end{equation}
At inference, since \(\bm{x}\) is unknown, we approximate it with a learned network \(\bm{x}_\theta(\bm{z}_t,t)\), trained with the objective described in~\Cref{app:more_details_on_mdlms}. The learned time reversal is then \(p_\theta(\bm{z}_s | \bm{z}_t)=q\bigl(\bm{z}_s | \bm{z}_t,\bm{x}_\theta(\bm{z}_t,t)\bigr)\). The network prediction \(\bm{x}_\theta\) is usually constrained so that (\romannumeral1) it assigns zero probability to the mask token, and (\romannumeral2) it directly copies already unmasked tokens, ensuring that only masked tokens need to be reconstructed.

\paragraph{Na\"ive parallel sampling.} A simple way to sample from MDLMs is to unmask positions parallelly at each step. For sequences of length \(L\), this corresponds to a factorized reverse transition:
\begin{equation}
\label{eq:naive_parallel_sampling}
p_\theta(\bm{z}_s^{1:L}| \bm{z}_t^{1:L})
\coloneqq\prod_{\ell=1}^{L} p_\theta(\bm{z}_s^{\ell}| \bm{z}_t^{1:L})
=\prod_{\ell=1}^{L} q\bigl(\bm{z}_s^{\ell}| \bm{z}_t^{1:L},\bm{x}_\theta(\bm{z}_t^{1:L},t)\bigr).
\end{equation}
Due to the small changes between adjacent diffusion steps, only a small subset of entries are effectively unmasked at each step, leading to inefficient use of computation.
 
\paragraph{First hitting sampling (FHS).} To mitigate the inefficiency of na\"ive parallel sampling, \citet{zheng24masked} proposed first hitting sampling (FHS), which only simulates the \emph{moments when actual unmasking events occur}. Set \(\tau_L=1\). When \(n\) masked tokens remain, a uniform random variable \(u_n \sim \mathcal{U}(0,1)\) is used to sample the next event time
\begin{equation}
\tau_{n-1} = \alpha^{-1}\bigl(1-u_n^{1/n}(1-\alpha_{\tau_n})\bigr),
\end{equation}
where \(\alpha^{-1}\) is the inverse noise schedule and \(\tau_n\) is the current time. One of the masked positions is then chosen uniformly at random and unmasked according to the model's categorical prediction at time \(\tau_{n-1}\). See~\Cref{alg:unmask-branch} in~\Cref{subsec:branching_via_unmasking_for_increased_branch_diversity} for the pseudocode.

\subsection{Test-Time Alignment}  

A common objective in TTA is to obtain generations that remain consistent with the pretrained model while maximizing a task-specific reward \(r\colon \mathcal{V}^{L} \rightarrow \mathbb{R}\). This can be expressed as a KL-regularized optimization problem with a closed-form solution~\citep{faria2025sample,uehara2024understanding}:
\begin{equation}
\begin{aligned}
p_{\mathrm{tar}} = \arg\max_{p}
\mathbb{E}_{\bm{x}^{1:L} \sim p} [ r(\bm{x}^{1:L})]
- \lambda D_{\mathrm{KL}}(p \| p_{\mathrm{pre}})\propto 
  p_{\mathrm{pre}}(\bm{x}^{1:L})\cdot \exp\Bigl(\dfrac{r(\bm{x}^{1:L})}{\lambda}\Bigr).
\end{aligned}
\end{equation}
where \(\lambda > 0\) is a temperature parameter that controls the trade-off between reward maximization and staying close to the pretrained distribution. Following the entropy-regularized inference-time alignment framework~\citep{pani2025test,singhal2025general,uehara2024bridging}, the corresponding \emph{soft value function} at step \(t\) is defined as
\begin{equation}
v_{t}(\bm{z}_t^{1:L})
\coloneqq \lambda \log
\mathbb{E}_{\bm{x}^{1:L} \sim p_{\mathrm{pre}}(\cdot|\bm{z}_t^{1:L} )}
\Bigl[\exp\Bigl(\dfrac{r(\bm{x}^{1:L})}{\lambda}\Bigr)
\Bigr]\approx r\bigl(\mathbb{E}_{\bm{x}^{1:L} \sim p_{\mathrm{pre}}(\cdot|\bm{z}_t^{1:L})}[\bm{x}^{1:L}]\bigr).
\end{equation}
The last approximation involves two simplifications: replacing the log–exp expectation with a direct expectation, and evaluating the reward at the mean instead of averaging. This reduces the problem to approximating the soft value function at intermediate steps rather than computing it exactly.

\section{Method}
\label{sec:method}

We introduce \textsc{TReASURe}, a TTA method for MDLMs based on tree search. Classical tree search consists of two components: (\romannumeral1) branching, which expands the search frontier by generating diverse candidate continuations, and (\romannumeral2) pruning, which retains only the most promising nodes using a task-specific value function. We first show how this framework applies naturally to \emph{continuous} diffusion models, then highlight the unique challenges in MDLMs, and describe how \textsc{TReASURe} rethinks branching and pruning to address them.

\subsection{Rethinking Branching and Pruning for MDLMs}

\paragraph{Tree search for continuous diffusion models.} In continuous diffusion models, tree search is a natural fit~\citep{li2025dynamic}. Each reverse step \(p_{\theta}(\bm{z}_{t-1}|\bm{z}_t)\) is stochastic (e.g., in DDPM~\citep{ho2020denoising}), so branching arises naturally by sampling multiple candidates for \(\bm{z}_{t-1}\). Because the latent space is smooth, the model's prediction \(\hat{\bm{x}}_0(\bm{z}_t,t)\) (or short rollouts) yields reliable intermediate reward estimates \(r(\hat{\bm{x}}_0(\bm{z}_t,t))\), enabling effective pruning by discarding low-value branches.

\begin{wrapfigure}{R}{0.5\textwidth}
\vspace{-1em}
\centering
\includegraphics[width=0.48\textwidth]{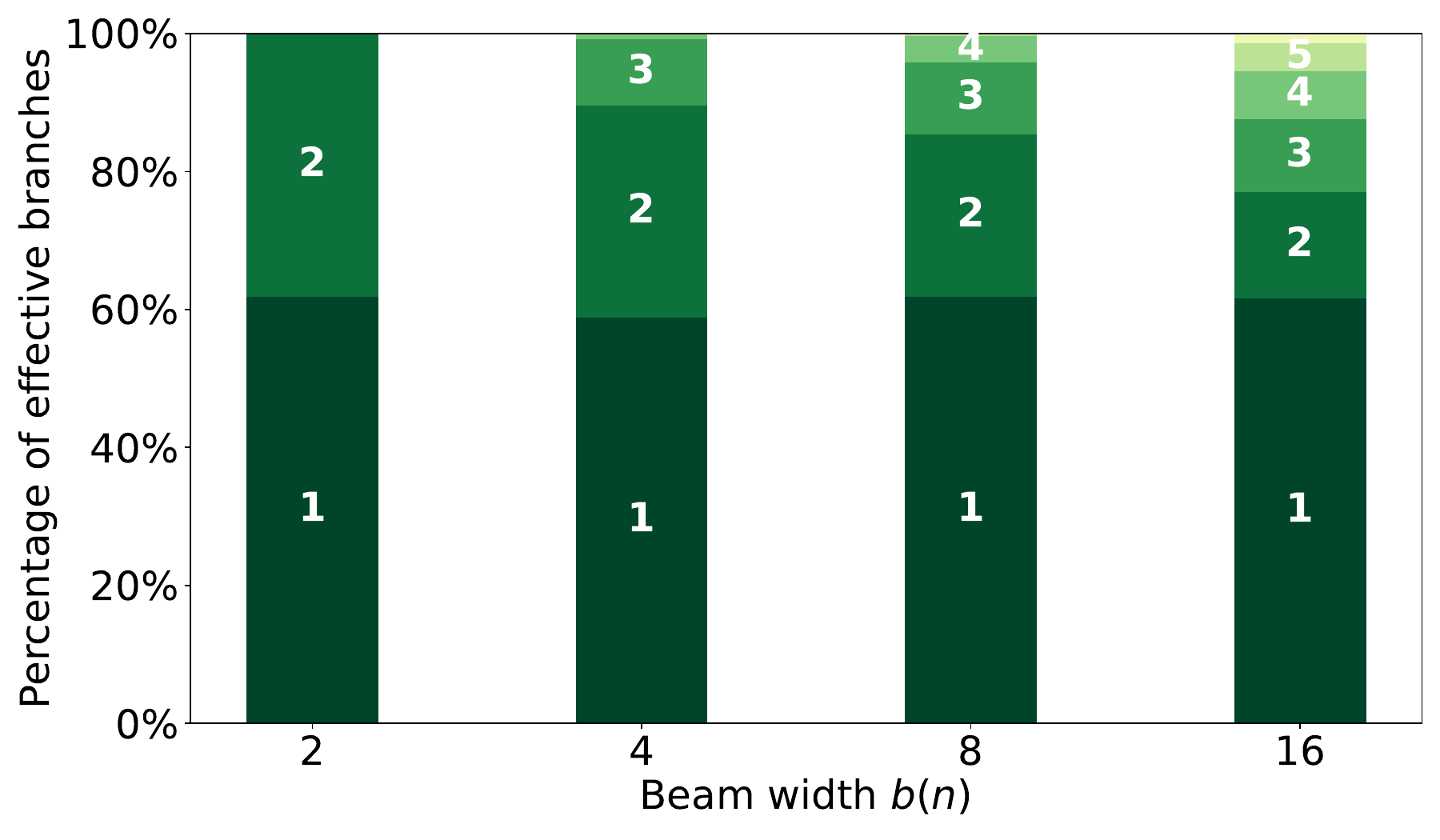}
\caption{\textbf{Distribution of effective branch counts under different beam widths \(b(n)\).} Despite wider beams, most nodes collapse to only one or two effective (distrinct) branches, indicating that parallel sampling produces many redundant branches.}
\label{fig:branch_motivation}
\vspace{-1em}
\end{wrapfigure}

\paragraph{Challenges for branching in MDLMs.} The key challenge is that branching in MDLMs does not naturally produce diverse trajectories. This is due to two structural properties of masked diffusion. First, sampling is performed \emph{in parallel}: all masked positions are updated simultaneously, resulting in tightly coupled token distributions. Repeated sampling from the same state yields highly correlated or nearly identical candidates, so naive branching by repeating the sampling multiple times explores only a narrow subset of the space, as illustrated in~\cref{fig:branch_motivation}. Second, the \emph{unmasking schedule} which determines which tokens are revealed at each step is decided endogenously by the model (e.g., via confidence thresholds). Since this schedule is unpredictable, local resampling rarely alters which tokens are committed next, limiting trajectory diversity. These phenomena make straightforward token-level branching ineffective, requiring a rethinking of how to construct and expand the search tree.

\begin{wrapfigure}{R}{0.5\textwidth}
\vspace{-2.7em}
\centering
\includegraphics[width=0.48\textwidth]{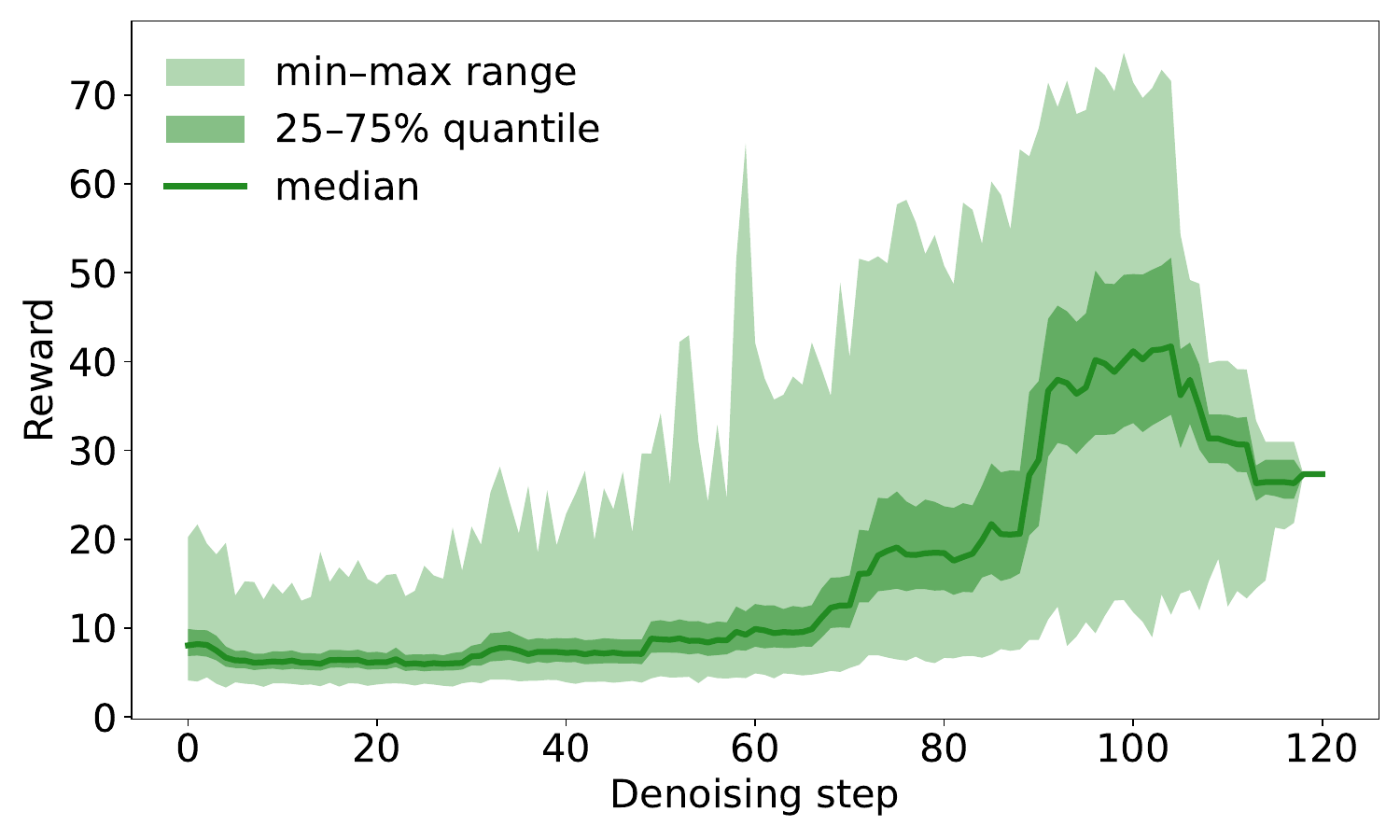}
\caption{\textbf{Reward variation across denoising steps on CoLA.} Median (line), interquartile range (25--75\%, dark green), and min--max range (light green) reveal large fluctuations, underscoring the need for stable pruning rules.}
\label{fig:prune_motivation}
\vspace{0.2em}
\end{wrapfigure}

\paragraph{Challenges for pruning in MDLMs.} Unlike continuous diffusion models, which predict a point in latent space, MDLMs output a \emph{distribution} over vocabulary tokens at each masked position. Estimating the value function thus requires sampling a complete sequence from this distribution,
\begin{equation}
\hat{v}_t(\cdot) \coloneqq r\bigl(\bm{\hat{x}}_0^{1:L}(\bm{z}_t^{1:L})\bigr),
\end{equation}
where \(\bm{\hat{x}}_0^{1:L}\) denotes a random sample from \(\bm{x}_\theta(\bm{z}_t^{1:L}, t)\)~\citep{singhal2025general}. This introduces high variance: small perturbations in \(\bm{\hat{x}}_0^{1:L}\) can cause large changes in the resulting reward. In principle, this variance could be reduced by sampling multiple completions per node, but doing so greatly increases computational cost. As a result, pruning becomes unstable, and heuristics that work well in continuous diffusion often fail in the discrete masked setting. \Cref{fig:prune_motivation} illustrates this effect on the CoLA reward~\citep{warstadt2019neural}, where reward values fluctuate widely across denoising steps. 
Additional results for toxicity and sentiment rewards are provided in~\Cref{app:more_experiments}, highlighting the generality of this challenge and motivating the need for alternative pruning strategies.

\subsection{Branching via Unmasking for Increased Branch Diversity}
\label{subsec:branching_via_unmasking_for_increased_branch_diversity}

Recall that MDLMs pose unique branching difficulties: (\romannumeral1) parallel updates couple all masked tokens, so repeated resampling produces highly similar candidates; and (\romannumeral2) the unmasking schedule is endogenous, so perturbing probabilities rarely alters which token is committed next.

\begin{wrapfigure}{R}{0.5\textwidth}
\small
\vspace{-2.8em}
\begin{minipage}{0.5\textwidth}
\begin{algorithm}[H]
\caption{\textsc{UnmaskBranch}}
\label{alg:unmask-branch}
\begin{algorithmic}[1]
\Require Parent node \((\bm{z}_n^{1:L},\tau_n)\) with \(n\) masks; beam width \(b(n)\)
\State \(u_n \sim \mathcal{U}(0,1)\)
\State \(\tau_{n-1} \gets \alpha^{-1}\bigl(1-u_n^{1/n}(1-\alpha_{\tau_n})\bigr)\)
\State \(\bm{\mu}_n \gets \bm{x}_\theta(\bm{z}_n^{1:L},\tau_{n-1})\)
\State \(\ell \sim \mathrm{Unif}\{j\colon\bm{z}_n^{(j)}=\bm{m}\}\)
\State \(\mathcal{Z} \gets \mathrm{TopK}_{b(n)}(\bm{\mu}_n^{\ell})\)
\Comment{\(\mathcal{Z}\) stores all the selected tokens}
\State \Return \((\mathcal{Z}, \tau_{n-1}, \bm{\mu}_n)\)
\end{algorithmic}
\end{algorithm}
\end{minipage}
\vspace{-1em}
\end{wrapfigure}

To resolve these challenges, we branch only at commitment events. Using first-hitting sampling~\citep{zheng24masked}, from the current unmasking time \(\tau_n\), we jump directly to the next unmasking time \(\tau_{n-1}\), evaluate the model once, and generate child nodes by (\romannumeral1) uniformly selecting a masked index and (\romannumeral2) enumerating the top-\(b(n)\) tokens at that index (see~\Cref{alg:unmask-branch} for the pseudocode). This design introduces diversity in both unmasking \emph{order} and token \emph{content}, avoids wasted updates, and requires only one model call per parent node. The beam width \(b(n)\) flexibly controls exploration, enabling wide yet efficient search at test time.

In comparison, na\"ive parallel sampling must repeatedly simulate transitions until unmasking occurs. We show in~\Cref{thm:unmaskbranch-cost} that to obtain \(b(n)\) child nodes from a parent node with \(n\) masked tokens, it requires an average of \(b(n)/(1-\exp(-nh))\) model evaluations, where \(h \ll 1\) is the discretization step. This cost grows rapidly as \(h \to 0\), which is often necessary for accurate sampling~\citep{sahoo2024simple}. In contrast, our method performs only one model evaluation per parent node, yielding substantial speedup in branching. 

\begin{restatable}[Efficiency of \textsc{UnmaskBranch} over na\"ive parallel sampling]{theorem}{UnmaskBranch}
\label{thm:unmaskbranch-cost}
Fix a parent node with \(n\) masked tokens and reparameterize time by \(\gamma(t)=-\log(1-\alpha_t)\). Discretize \(\gamma\) on a uniform grid with step size \(h\ll 1\). Run na\"ive parallel sampling (\cref{eq:naive_parallel_sampling}) and, in each run, stop at the first branch commitment (i.e., the first position that unmasks); repeat from the same parent node until \(b(n)\) child nodes have been obtained. Then the expected total number of model evaluations required by the na\"ive parallel sampling is 
\begin{equation}
\mathbb{E}[\mathrm{evals}]=\dfrac{b(n)}{1-\exp(-nh)}. 
\end{equation} 
In contrast, \textsc{UnmaskBranch} produces \(b(n)\) child nodes with exactly one evaluation. In particular, for \(b(n)=1\), \textsc{UnmaskBranch} matches the na\"ive-parallel first-change distribution over \((\tau_{n-1},\ell)\), where \(\tau_{n-1}\) is the next unmasking time and \(\ell\in[n]\) is the index of the committed position.
\end{restatable}

Beyond efficiency, we also consider the distributional behavior of \textsc{UnmaskBranch} in~\Cref{thm:unmaskbranch-cost}. For \(b(n)=1\), it produces the same distribution over the next unmasking time \(\tau_{n-1}\) and committing index \(\ell \in [n]\) as the first-change outcome of na\"ive parallel sampling, as originally shown by~\citet[Proposition 4.1]{zheng24masked}. For completeness, we restate the result and provide an alternative proof in~\Cref{app:theoretical_guarantee_Of_treasure}. For \(b(n)>1\), however, na\"ive parallel sampling may unmask different positions across runs, whereas \textsc{UnmaskBranch} fixes a single index, leading to a different distribution.

\subsection{Pruning via Resubstitution for Efficient Reward Evaluation}

\begin{wrapfigure}{R}{0.5\textwidth}
\small
\vspace{-2.4em}
\begin{minipage}{0.5\textwidth}
\begin{algorithm}[H]
\caption{\textsc{ResubstituteScore}}
\label{alg:resubtitute-score}
\begin{algorithmic}[1]
\Require Candidate \(\bm{z}_{n-1}^{1:L}\), probabilities \(\bm{\mu}_n\)
\State \(\hat{\bm{x}}_0^{1:L} \gets \bm{z}_{n-1}^{1:L}\)
\For{masked position \(\ell\) in \(\bm{z}_{n-1}^{1:L}\)}
  \State \(\hat{\bm{x}}_0^\ell \gets \arg\max \bm{\mu}_n^\ell\) \Comment{let \(\hat{\bm{x}}_0^\ell\) be the token with the highest probability}
\EndFor
\State \Return \(r(\hat{\bm{x}}_0^{1:L})\)
\end{algorithmic}
\end{algorithm}
\end{minipage}
\vspace{-1.2em}
\end{wrapfigure}

Pruning is equally challenging: (\romannumeral1) categorical predictions make sampled rewards noisy and unstable; and (\romannumeral2) drawing extra completions per node inflates compute cost. As a result, naive scoring destabilizes search.

To address these challenges, we reuse the probabilities from branching to construct a provisional completion by \emph{resubstitution}: committed tokens remain fixed, while masked positions are filled with current head predictions. This single proxy completion is scored once with the reward (see~\Cref{alg:resubtitute-score} for the pseudocode). Resubstitution enables low-variance, deterministic scoring without extra model calls; it uses temporal locality by evaluating probabilities at the precise commitment time; and it ensures pruning costs to scale linearly with the number of parent nodes \(m(n)\). Together, these properties enable stable, reward-aware pruning under the same NFE budget as baseline sampling.

We provide a theoretical justification for this pruning rule. Assuming the reward is Hamming--Lipschitz (\Cref{assump:lipschitz}), we show in \Cref{thm:resubstitute-gap} that the gap between the resubstituted reward and the true expected reward is bounded by the model's predictive uncertainty at masked positions.

\begin{assumption}[Hamming--Lipschitz reward]
\label{assump:lipschitz}
We assume that the reward function is Lipschitz continuous with respect to the Hamming distance, i.e., there exists a constant \(\beta > 0\), such that for all \(\bm{x}^{1:L},\bm{y}^{1:L}\in\mathcal{V}^L\), we have
\begin{equation}
\bigl|r(\bm{x}^{1:L})-r(\bm{y}^{1:L})\bigr|\leq\beta\cdot d_{\mathrm{H}}(\bm{x}^{1:L},\bm{y}^{1:L}),
\end{equation}
where the Hamming distance is given by \(d_{\mathrm{H}}(\bm{x}^{1:L},\bm{y}^{1:L})\coloneqq\sum_{\ell=1}^{L}\mathbf{1}_{\{\bm{x}^\ell\neq \bm{y}^\ell\}}\).
\end{assumption}

\begin{restatable}[Resubstitution gap controlled by max confidence]{theorem}{ResubGap}
\label{thm:resubstitute-gap}
Let \((\bm{z}^{1:L}_{n-1},\tau_{n-1})\) be a state, and let \(\bm{\mu}_n=\bm{x}_\theta(\bm{z}^{1:L}_{n-1},\tau_{n-1})\) denote the model probabilities. Denote by \(\mathcal{I}_{n-1}=\{\ell\colon\bm{z}^{\ell}_{n-1}=\bm{m}\}\) the set of masked indices. Let \(\hat{\bm{x}}_0^{1:L}\) be the resubstituted completion from~\Cref{alg:resubtitute-score}, and let \(\bm{X}^{1:L}\) be a random completion obtained by sampling each masked index \(\ell\in\mathcal{I}_{n-1}\) as \(X^\ell\sim \mathrm{Cat}(\cdot;\bm{\mu}_n^\ell)\). Under~\Cref{assump:lipschitz} and let \(\beta\) be the Lipschitz constant therein, we have
\begin{equation}
\Bigl|\mathbb{E}\bigl[r(\bm{X}^{1:L})\bigr] - r\bigl(\hat{\bm{x}}_0^{1:L}\bigr) \Bigr|
 \le\
\beta \sum_{\ell\in\mathcal{I}_{n-1}}\Bigl(1-\max_{v\in[V]}\bm{\mu}_n^\ell(v)\Bigr).
\end{equation}
\end{restatable}

The bound in~\Cref{thm:resubstitute-gap} justifies using resubstitution for pruning, as it shows that the approximation error is directly bounded by predictive uncertainty: when the model assigns high confidence to its top prediction, the resubstitution reward is close to the expected reward.

\subsection{Full \textsc{TReASURe} Algorithm}

Having introduced the two key building blocks \textsc{UnmaskBranch} (\Cref{alg:unmask-branch}) for branching and \textsc{ResubstituteScore} (\Cref{alg:resubtitute-score}) for pruning, we now describe how \textsc{TReASURe} integrates them into a complete tree-search procedure, summarized in~\Cref{alg:treasure}. Moreover, \Cref{thm:monotone-m} guarantees that increasing the tree width \(m(n)\) always improves the final reward.

\begin{restatable}[Reward monotonicity in tree width]{theorem}{MonotoneM} \label{thm:monotone-m} Fix the beam width \(b(\cdot)\) and run \textsc{TReASURe} (\Cref{alg:treasure}) twice with tree-width schedules \(m(\cdot)\) and \(m'(\cdot)\) such that \(m'(n)\ge m(n)\) for all \(n\in\{1,\dotsc,L\}\). Couple all randomness across the two runs (same \(\textsc{UnmaskBranch}\) draws and model outputs), and use the same deterministic tie-breaking in \(\mathrm{TopK}\). Let the returned rewards be \(r_\star(m)\) and \(r_\star(m')\). Then \(r_\star(m') \geq r_\star(m)\). \end{restatable}

\Cref{thm:monotone-m} shows that increasing the tree width \(m(n)\), which determines the number of model evaluations, leads to improved final rewards. This trend is verified empirically in~\Cref{sec:experiments}. As a remark, increasing the beam width \(b(n)\) yields stronger local candidates, but this does not guarantee monotonic improvement, since locally better branches do not necessarily lead to higher final rewards.

\begin{algorithm}[htb]
\small
\caption{Tree Reward-Aligned Search with Unmasking and Resubstitution (\textsc{TReASURe})}
\label{alg:treasure}
\begin{algorithmic}[1]
\Require Pretrained MDLM \(\bm{x}_\theta(\bm{z}^{1:L}, t)\); reward \(r(\cdot)\); length \(L\); beam width \(b(\cdot)\); tree width \(m(\cdot)\)
\Ensure Final sequence \(\bm{z}^{1:L}_\star\) and reward \(r_\star\)
\State \(\tau_L \gets 1\), \quad \(\bm{z}_L \gets [\bm{m},\dots,\bm{m}]\)
\State \(\mathcal{S} \gets \{(\bm{z}_L^{1:L},\tau_L)\}\)
\For{\(n=L\) \textbf{down to} \(1\)}
  \State \(\mathcal{C} \gets \varnothing\) \Comment{candidate pool}
  \ForAll{\((\bm{z}_n^{1:L},\tau_n)\in\mathcal{S}\)}
    \State \((\mathcal{Z},\tau_{n-1},\bm{\mu}_n) \gets \Call{UnmaskBranch}{\bm{z}_n^{1:L},\tau_n,b(n)}\) \Comment{call~\Cref{alg:unmask-branch}}
    \ForAll{\(\bm{z}_{n-1}^{1:L}\in\mathcal{Z}\)}
      \State \(r \gets \Call{ResubstituteScore}{\bm{z}_{n-1}^{1:L},\bm{\mu}_n}\) \Comment{call~\Cref{alg:resubtitute-score}}
      \State \(\mathcal{C} \gets \mathcal{C} \cup \{(\bm{z}_{n-1}^{1:L},\tau_{n-1},r)\}\)
    \EndFor
  \EndFor
  \State \(\mathcal{S} \gets \mathrm{TopK}_{m(n)}(\mathcal{C})\) \Comment{keep best \(m(n)\) nodes}
\EndFor
\State \((\bm{z}^{1:L}_\star, \tau_0, r_\star) \gets \arg\max_{(\bm{z}_0^{1:L},\tau_0,r)\in\mathcal{S}} r\)
\State \Return \(\bm{z}^{1:L}_\star\), \(r_\star\)
\end{algorithmic}
\end{algorithm}

\section{Experiments}
\label{sec:experiments}

In this section, we evaluate the performance of \textsc{TReASURe} on TTA for controllable text generation. All experiments are run on a single NVIDIA A100 GPU. Further experimental settings and additional results are provided in~\Cref{app:experimental_details} and~\Cref{app:more_experiments}.

\subsection{Controllable Text Generation}

\paragraph{Experimental settings.} 
We adopt the MDLM implementation by~\citet{sahoo2024simple} as our base model. To ensure a fair comparison across different TTA approaches, we follow the experimental protocol described in~\citet{han2023ssd,singhal2025general} and fix the number of denoising steps to \(1{,}000\) for all the baseline methods. For each controllable prompt, we generate \(20\) continuations of length \(128\) tokens using \(15\) prompts, consistent with prior work. We report the NFE, defined as the total number of forward passes through the pretrained MDLM during generation, as the primary measure of test-time compute cost. Following prior work, we ignore the computational cost of the reward model, as it is shared across all methods and does not affect relative efficiency comparisons. More implementation details and hyperparameter configurations are provided in~\Cref{app:experimental_details}.

\paragraph{Baselines.} 
To evaluate the effectiveness of \textsc{TReASURe}, we compare it against several representative TTA approaches:
\begin{itemize}[leftmargin=*]
\item \textbf{Base model sampling:} This baseline directly generates candidate outputs from a pre-trained MDLM without applying any additional optimization.
\item \textbf{Best-of-\(N\) (BoN):} This approach samples \(N\) candidate sequences from the same pre-trained model and selects the one achieving the highest reward.
\item \textbf{Feynman--Kac (FK) steering~\citep{singhal2025general}:} 
A Sequential Monte Carlo (SMC)-based approach that maintains particles during generation and resamples them by reward-weighted importance, thereby improving sample efficiency and alignment.
\end{itemize}

\paragraph{Rewards and metrics.} 
For controllable text generation, we consider four downstream reward functions:
\textbf{(\romannumeral1)} \textbf{Perplexity} (\textsc{Gen}.~\textsc{PPL}): Computed using GPT-2~\citep{radford2019language}, this metric encourages generations that are more likely under a pretrained language model.
\textbf{(\romannumeral2)} \textbf{Linguistic acceptability} (\textsc{CoLA}): Based on a classifier~\citep{morris2020textattack} trained on the CoLA dataset~\citep{warstadt2019neural}, this reward favors sentences that are grammatically well-formed.
\textbf{(\romannumeral3)} \textbf{Toxicity score} (\textsc{Toxicity}): Using a toxicity detection classifier~\citep{logacheva2022paradetox} trained to identify harmful or offensive content, this reward assesses model vulnerabilities and penalizes toxic outputs.
\textbf{(\romannumeral4)} \textbf{Sentiment score} (\textsc{Sentiment}): Leveraging a sentiment classifier~\citep{barbieri2020tweeteval} trained on social media data, this reward guides the model toward producing outputs with the desired sentiment (e.g., positive). We evaluate model performance using the aforementioned reward functions as evaluation metrics, and additionally measure diversity (deferred to~\Cref{app:more_experiments}) for a comprehensive assessment.

\subsection{Experimental Results}

\paragraph{Comparison with baseline methods.} 
We measure compute in terms of NFE. For BoN and FK-steering, following prior work, we fix the denoising steps at \(T=1{,}000\) and vary the per-step NFE across \(\{1, 2, 4, 6, 8, 16\}\), leading to total budgets of \(1{,}000 \times \{1, 2, 4, 6, 8, 16\}\) evaluations. Note that in \textsc{TReASURe}, the number of denoising steps naturally scales with sequence length (\(128\) tokens in our setting), so total budgets are smaller but directly comparable in terms of per-step NFE. \Cref{tab:main} reports results on controllable text generation across four reward functions: \textsc{CoLA}, \textsc{Toxicity}, \textsc{Sentiment}, and \textsc{Gen}.~\textsc{PPL}. \textsc{TReASURe} consistently achieves state-of-the-art performance across all tasks. In the low-NFE regime, where compute is scarce, it already surpasses the best baseline results obtained at much higher NFEs. As NFE increases, performance improves steadily, as predicted by \Cref{thm:monotone-m}, and \textsc{TReASURe} dominates all baselines at every budget.

\begin{table}[t]
\centering
\small
\renewcommand{\arraystretch}{1} 
\setlength{\tabcolsep}{8pt}       
\begin{tabular}{lccccc}
\toprule
\textsc{Method} & \textsc{NFE} & \textsc{CoLA} \(\uparrow\) & \textsc{Toxicity} \(\uparrow\) & \textsc{Sentiment} \(\uparrow\) & \textsc{Gen}.~\textsc{PPL} \(\downarrow\) \\
\midrule
MDLMs & \(1\) & \(25.56\) & \(0.89\) & \(12.44\) & \(80.58\) \\
\midrule
\multirow{5}{*}{BoN}
    & \(2\) & \(43.16\) & \(0.70\) & \(22.11\) & \(70.58\) \\
    & \(4\) & \(63.57\) & \(2.44\) & \(32.89\) & \(55.47\) \\
    & \(6\) & \(65.96\) & \(2.57\) & \(45.50\) & \(52.23\) \\
    & \(8\) & \(73.11\) & \(6.67\) & \(48.44\) & \(47.91\) \\
    & \(16\) & \(77.56\) & \(11.11\) & \(65.56\) & \(42.46\) \\
\midrule
\multirow{5}{*}{FK-steering} 
    & \(2\) & \(45.96\) & \(1.05\) & \(20.00\) & \(66.79\) \\
    & \(4\) & \(66.62\) & \(1.56\) & \(36.33\) & \(56.36\) \\
    & \(6\) & \(69.82\) & \(3.16\) & \(41.05\) & \(50.16\) \\
    & \(8\) & \(72.67\) & \(4.00\) & \(49.33\) & \(46.60\) \\
    & \(16\) & \(76.44\) & \(9.67\) & \(61.33\) & \(41.56\) \\
\midrule
\multirow{5}{*}{\textsc{TReASURe}}  
    & \(2\) & \(77.67\) & \(64.00\) & \(98.67\) & \(15.37\) \\
    & \(4\) & \(84.22\) & \(93.33\) & \(98.90\) & \(9.22\) \\
    & \(6\) & \(89.19\) & \(96.60\) & \(99.11\) & \(7.60\) \\
    & \(8\) & \(93.33\) & \(100.00\) & \(100.00\) & \(6.60\) \\
    & \(16\) & \(\textbf{98.35}\) & \(\textbf{100.00}\) & \(\textbf{100.00}\) & \(\textbf{5.11}\) \\
\bottomrule
\end{tabular}
\caption{\textbf{Main results on TTA for controllable text generation with MDLMs.} We compare the base MDLM, Best-of-\(N\) (BoN), FK-steering, and \textsc{TReASURe} (ours) across four reward functions: \textsc{CoLA}, \textsc{Toxicity}, \textsc{Sentiment}, and \textsc{Gen}.~\textsc{PPL}. \textsc{TReASURe} consistently outperforms all baselines, with especially large gains in low-NFE regimes and continued improvements as NFE increases. Arrows indicate whether higher (\(\uparrow\)) or lower (\(\downarrow\)) values are preferred. We remark that while lower \textsc{Toxicity} is generally more desirable, increasing it also serves as a valid TTA target for benchmarking purposes.}
\label{tab:main}
\end{table}

\paragraph{Visualization of reward trajectories.} To understand how TTA interacts with MDLM denoising under \textsc{TReASURe}, we plot reward trajectories across denoising steps from \(10\) independent trials, with each colored curve corresponding to one trial. As shown in \cref{fig:reward_trajectory}, rewards rise steadily during denoising. This reveals two properties of MDLM generation. First, the trajectory exhibits progressive refinement: early steps reconstruct coarse structures, while later ones refine them into coherent text. \textsc{TReASURe} exploits this by branching on alternative continuations and pruning low-reward ones. Second, the near-monotonic reward increase indicates that \textsc{TReASURe} not only improves final outputs but also steers intermediate states toward more desirable regions of the hypothesis space. Overall, \textsc{TReASURe} provides a stable mechanism for aligning MDLM outputs with reward signals throughout the denoising process.

\begin{figure}[htb]
\centering
\includegraphics[width=0.32\textwidth]{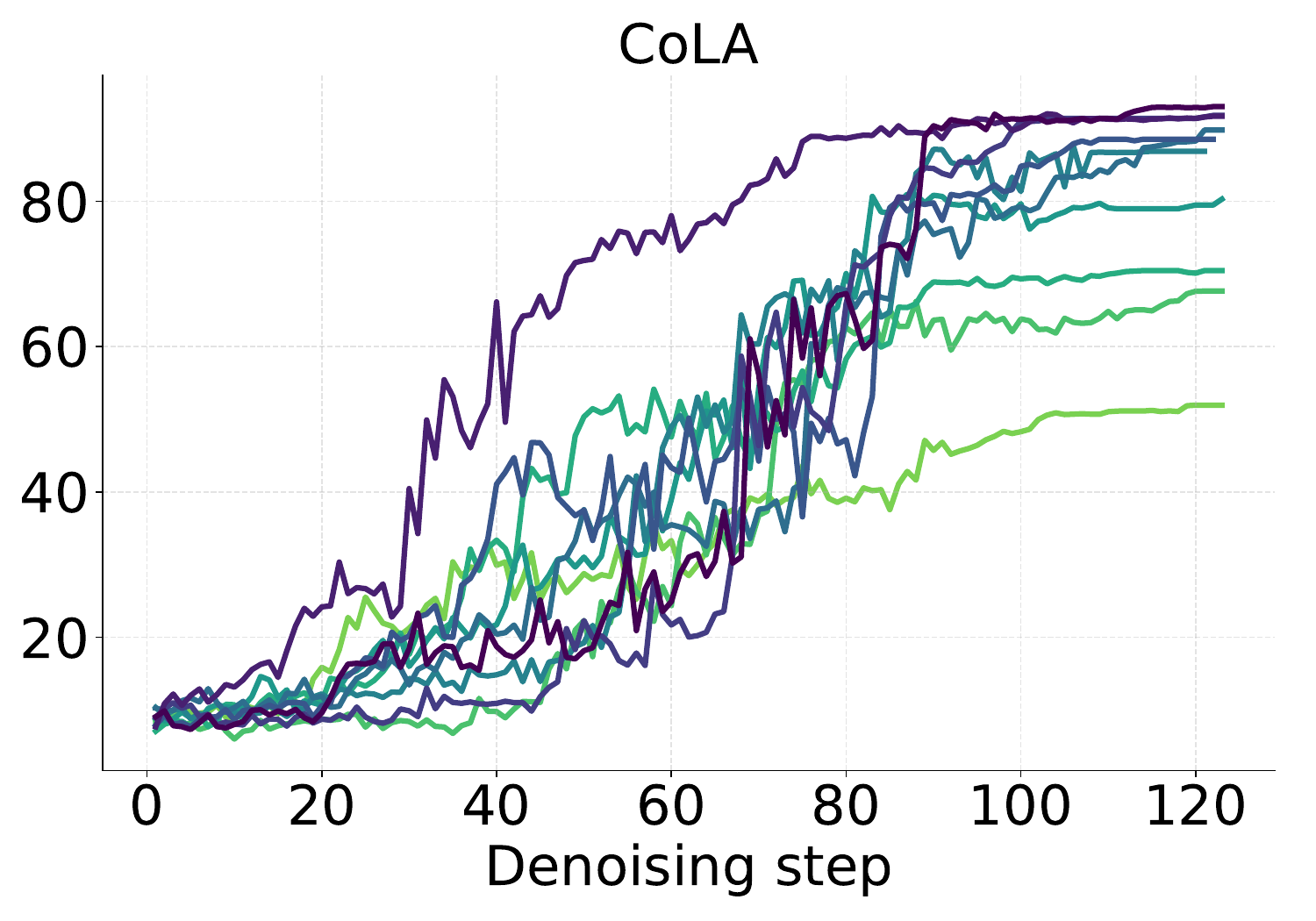}
\includegraphics[width=0.32\textwidth]{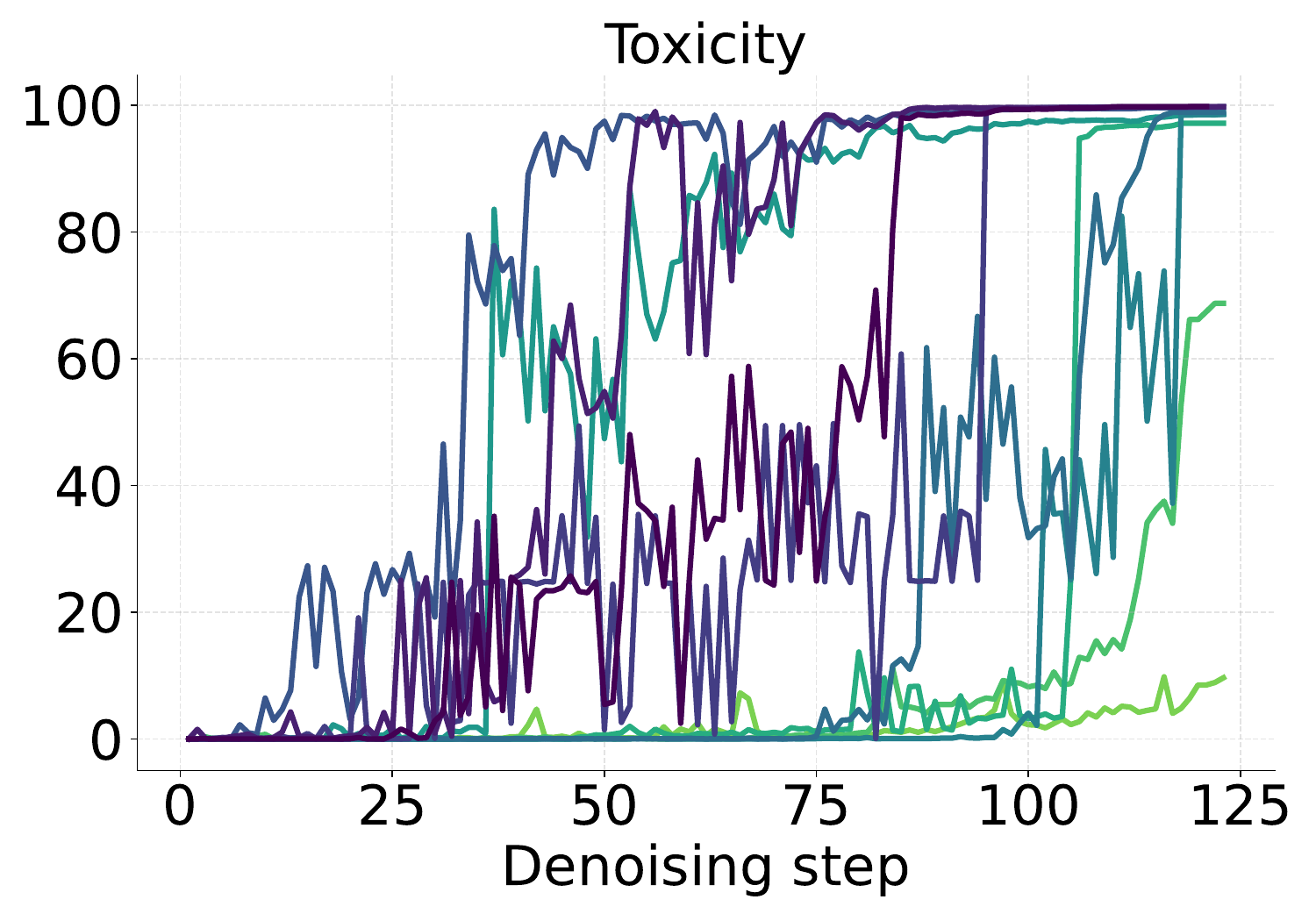}
\includegraphics[width=0.32\textwidth]{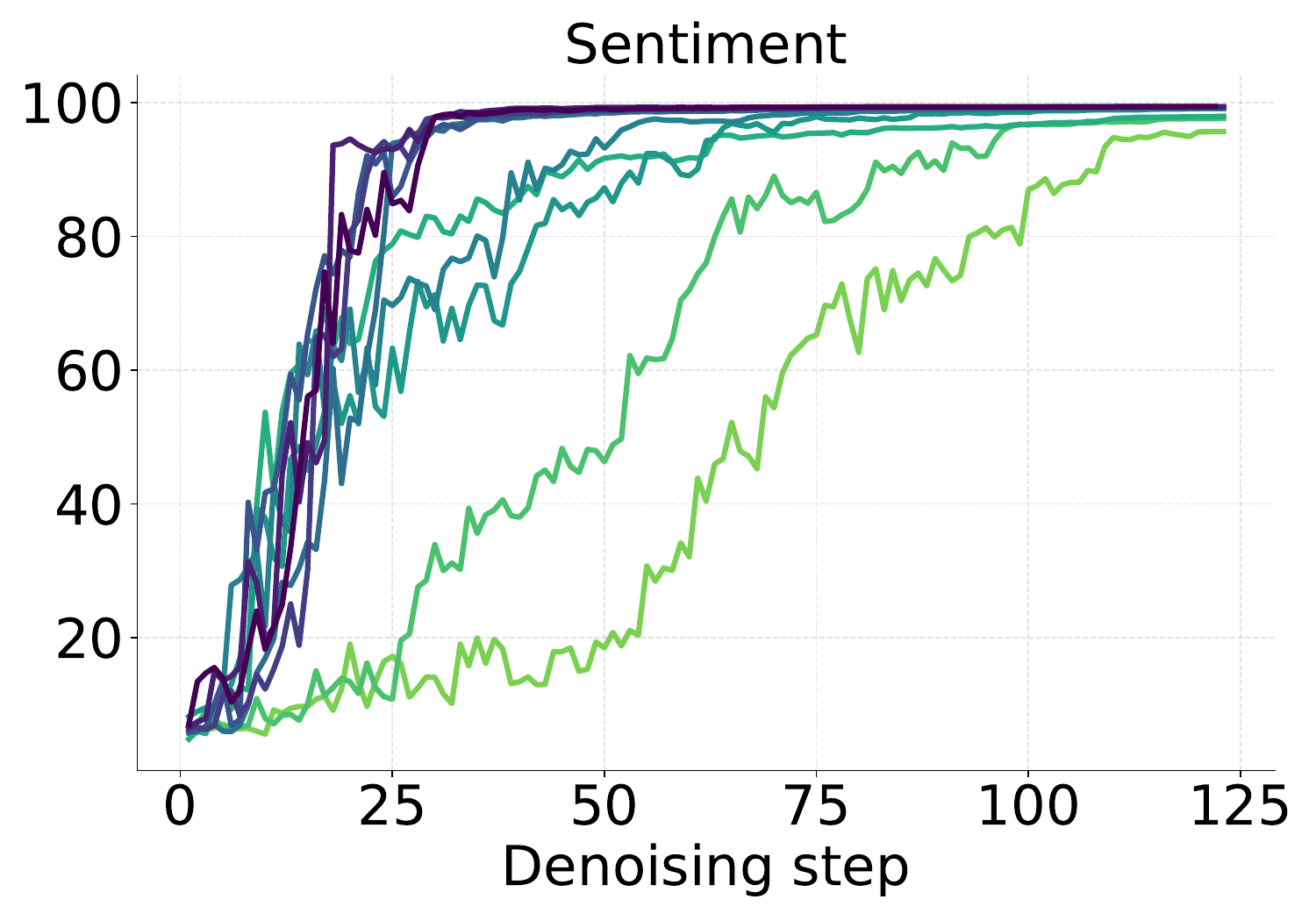}
\caption{\textbf{Reward trajectories across denoising steps over \(10\) independent trials.} We show the evolution of rewards for three task-specific objectives (\textsc{CoLA}, \textsc{Toxicity}, and \textsc{Sentiment}) (colors are sorted by final reward, but do not indicate temporal order). All three exhibit a generally increasing trend, illustrating progressive refinement during denoising and supporting the effectiveness of \textsc{TReASURe}. We omit \textsc{Gen}.~\textsc{PPL}, as perplexity is computed before the first \texttt{<EOS>} token and thus depends on variable sentence lengths, making step-wise comparisons inconsistent.}
\label{fig:reward_trajectory}
\vspace{-1em}
\end{figure}

\subsection{Ablation Studies}

\paragraph{Effectiveness of \textsc{ResubstituteScore}.}  
To validate the effectiveness of deterministic scoring in~\Cref{alg:resubtitute-score}, where the masked positions in \(\bm{z}_{n-1}^{1:L}\) are replaced by the \(\arg\max\) tokens of \(\bm{x}_\theta(\bm{z}_n^{1:L}, \tau_{n-1})\), we compare against two natural alternatives: (\romannumeral1) \emph{previous-step scoring}, which approximates the score using the posterior at the previous step, i.e., \(r(\bm{z}_n^{1:L})\). This corresponds to the approach of~\citet{singhal2025general}, but ignores information from the current step and yields high-variance estimates; and (\romannumeral2) \emph{true posterior scoring}, which estimates the score from the true posterior \(r(\bm{z}_{n-1}^{1:L})\) as in~\citet{chung2023diffusion}. While unbiased, this method requires additional NFEs per child node and thus incurs significant computational overhead. As shown in \Cref{tab:ablation}, \textsc{ResubstituteScore} consistently outperforms the first approach under matched NFE, and achieves performance comparable to the second approach while requiring substantially fewer evaluations.

\begin{table}[t]
\centering
\small
\renewcommand{\arraystretch}{1}
\setlength{\tabcolsep}{8pt}
\begin{tabular}{lccccc}
\toprule
\textsc{Method} & \textsc{NFE}  & \textsc{CoLA} $\uparrow$ & \textsc{Toxicity} $\uparrow$ & \textsc{Sentiment} $\uparrow$ & \textsc{Gen}.~\textsc{PPL} $\downarrow$ \\
\midrule
\rowcolor{gray!15}
& \(2\) & \(77.67\) & \(64.00\) & \(98.67\) & \(15.37\) \\
\rowcolor{gray!15}
\textsc{TReASURe} & \(4\) & \(84.22\) & \(93.33\) & \(98.90\) & \(9.22\) \\
\rowcolor{gray!15}
                    & \(6\) & \(89.19\) & \(96.60\) & \(99.11\) & \(7.60\) \\
\midrule
\multirow{3}{*}{\emph{previous-step scoring}}
  & \(2\) & \(45.01\) & \(26.67\) & \(93.33\) & \(60.14\) \\
  & \(4\) & \(74.44\) & \(68.89\) & \(95.10\) & \(20.81\) \\
  & \(6\) & \(73.33\) & \(73.94\) & \(97.27\) & \(20.03\) \\
\midrule
\multirow{3}{*}{\emph{true posterior scoring}}
  & \(10\;(=5\times2)\) & \(84.33\) & \(93.33\) & \(98.79\) & \(10.09\) \\
  & \(20\;(=5\times4)\) & \(93.33\) & \(100.00\) & \(100.00\) & \(7.30\) \\
  & \(48\;(=8\times6)\) & \(93.33\) & \(100.00\) & \(100.00\) & \(5.11\) \\
\bottomrule
\end{tabular}
\caption{\textbf{Ablation study on \textsc{ResubstituteScore}.} We compare the full \textsc{TReASURe} model (gray rows, copied from \Cref{tab:main}) against two variants: 
(\romannumeral1) \emph{previous-step scoring}, which reuses the reward from the prior step, and (\romannumeral2) \emph{true posterior scoring}, which evaluates the exact posterior but inflates NFE (e.g., \(10=5\;(\text{beam width})\times2\;(\text{tree width/NFE})\)). 
\textsc{ResubstituteScore} achieves the best trade-off, outperforming (\romannumeral1) under matched NFE and matching (\romannumeral2) with far fewer evaluations.}
\label{tab:ablation}
\end{table}

\paragraph{Effectiveness of beam width.} We further study the role of beam width while keeping the tree width fixed at \(4\), so that the per-step NFE remains constant. 
Unlike tree width (\Cref{thm:monotone-m}), larger beam width does not guarantee better rewards. As shown in \Cref{fig:beam_width}, the reward may fluctuate as beam width grows, highlighting that beam search alone cannot reliably ensure better alignment under fixed tree width.

\begin{figure}[htb]
\centering
\includegraphics[width=\linewidth]{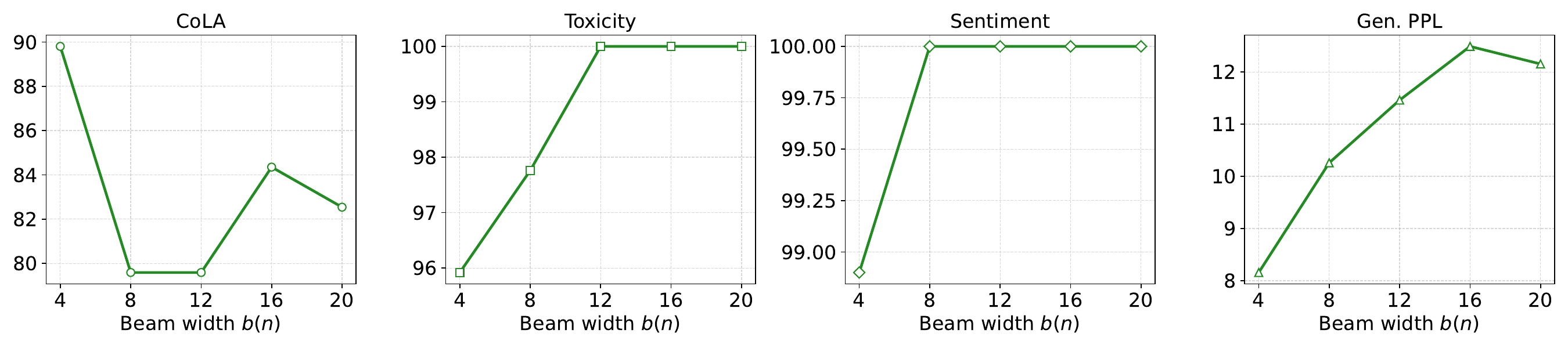}
\caption{\textbf{Effect of beam width on reward.} 
We fix the tree width to \(4\), keeping per-step NFE constant, and vary the beam width. In contrast to tree width (cf.~\Cref{thm:monotone-m}), increasing beam width does not guarantee better rewards, and performance may plateau or even degrade.}
\label{fig:beam_width}
\end{figure}

\section{Conclusion}

In this work, we proposed \textsc{TReASURe}, a tree-search method for test-time alignment in masked diffusion language models. It addresses the two central challenges for tree search in this setting: low-diversity branching under na\"ive parallel sampling and high-variance reward estimates due to distributional output. By introducing \textsc{UnmaskBranch}, which branches only at first-hitting unmask events, and \textsc{ResubstituteScore}, which prunes via deterministic resubstitution, \textsc{TReASURe} achieves stable and compute-efficient alignment. Our analysis characterizes why these strategies are reliable, and our experiments demonstrate strong gains across diverse controllable generation tasks. Future work includes extending the approach to long-context and multimodal MDLMs, and developing theoretically-grounded adaptive schedules that better balance exploration and efficiency.

\section*{Ethics Statement}

Our study uses only publicly available datasets and does not involve human participants. We are aware of potential misuse of machine learning models and emphasize that our contributions should be applied in a safe, fair, and responsible manner.

\section*{Reproducibility Statement}

The code for implementing \textsc{TReASURe} will be released publicly upon publication. Theoretical results are stated with assumptions specified and complete proofs provided in~\Cref{app:theoretical_guarantee_Of_treasure}. Datasets are standard public benchmarks, and details of preprocessing and experimental settings are included in~\Cref{app:experimental_details}.

\bibliography{iclr/iclr2026_conference}
\bibliographystyle{iclr/iclr2026_conference}

\clearpage
\appendix
\crefalias{section}{appendix}

\paragraph{Roadmap.} The appendix is organized as follows:
\begin{itemize}
\item \Cref{app:additional_literature_review} provides an additional literature review.
\item \Cref{app:more_details_on_mdlms} includes more details on MDLMs.
\item \Cref{app:theoretical_guarantee_Of_treasure} provides proofs to the theorems in~\Cref{sec:method}.
\item \Cref{app:experimental_details} reports additional experimental details omitted from the main text.  
\item \Cref{app:more_experiments} presents complementary experiments extending the main text.
\end{itemize}

\section{Additional Literature Review}
\label{app:additional_literature_review}

Recent progress in large-scale language modeling spans three key directions: (\romannumeral1) Autoregressive Language Models (ARMs) remain dominant but suffer from sequential generation constraints; (\romannumeral2) Masked Diffusion Language Models (MDLMs) offer a parallelizable alternative via coarse-to-fine refinement; and (\romannumeral3) Test-Time Alignment (TTA) methods aim to steer model outputs toward desired objectives without fine-tuning. We briefly review each direction below.

\paragraph{Autoregressive Language Models (ARMs).} ARMs have achieved unprecedented success in the era of large-scale language modeling, powering cutting-edge systems such as ChatGPT~\citep{achiam2023gpt}, DeepSeek~\citep{guo2025deepseek}, and the Qwen series~\citep{yang2025qwen3}, and driving significant advances toward Artificial General Intelligence (AGI). Following a causal next-token prediction paradigm, ARMs exhibit strong scaling properties and demonstrate impressive capabilities in reasoning~\citep{shao2024deepseekmath,o1}, planning~\citep{driess2023palm}, and multimodal understanding~\citep{liu2023visual,wang2024qwen2,chen2024internvl}.
However, their strict causal dependency introduces inherent limitations: generation is inherently sequential, inference remains computationally expensive, and controlling global properties such as structure and reasoning steps is challenging. This raises a natural question: \emph{Is AR the only viable paradigm for achieving AGI?} Recently, an increasing number of studies have explored Masked Diffusion Language Models (MDLMs)~\citep{sahoo2024simple} as an alternative framework, leveraging coarse-to-fine refinement and parallel decoding to rethink the foundations of large-scale language modeling.

\paragraph{Masked Diffusion Language Models (MDLMs).} Building on ideas from continuous diffusion models~\citep{gao2024flow,gao2025toward,ho2020denoising}, MDLMs stem from discrete diffusion models~\citep{austin2021structured,campbell2022continuous,lou2024discrete,sahoo2024simple,shi2024simplified} and demonstrate strong potential as an alternative to autoregressive paradigms. Closed-source systems such as Gemini Diffusion~\citep{gemini-thinking} and Mercury~\citep{labs2025mercury} achieve thousands of tokens per second, offering \(5\)--\(10\times\) faster generation than AR models of comparable size, highlighting the scalability and efficiency of the diffusion paradigm. On the open-source side, LLaDA~\citep{nie2025large} represents the first billion-scale MDLM trained from scratch on \(2.3\)T tokens, achieving performance competitive with LLaMA-3-8B~\citep{dubey2024llama} across reasoning, coding, and comprehension benchmarks. Building upon this, LLaDA-1.5~\citep{zhu2025llada} integrates reinforcement learning for preference alignment, further improving mathematical and code reasoning. In parallel, a continual pre-training paradigm adapts existing ARMs into MDLMs, with models such as DiffuLLaMA~\citep{gong2024scaling} or named DiffuGPT, and Dream-7B~\citep{ye2025dream} demonstrating strong performance by leveraging pretrained backbones like LLaMA~\citep{touvron2023llama} and Qwen2.5-7B~\citep{qwen2.5} while benefiting from the diffusion-native coarse-to-fine refinement process. This growing trend highlights MDLMs as a promising yet underexplored paradigm that enables parallel decoding, controllable refinement, and scalable training, motivating our investigation.

\paragraph{Test-Time Alignment (TTA).}
Existing TTA methods can be broadly categorized into two families: (\romannumeral1) sampling-based approaches, which guide generation by adjusting the sampling distribution, and (\romannumeral2) search-based strategies, which explicitly explore multiple decoding trajectories to identify high-reward outputs.
Among sampling-based methods, the most straightforward is Best-of-\(N\) (BoN)~\citep{stiennon2020learning,tang2024realfill}, a model-agnostic strategy that generates multiple candidates and selects the one achieving the highest reward.
However, BoN is costly for diffusion models due to their iterative denoising.
Recent work addresses this by integrating particle sampling into the generation process, allowing multiple candidates to be explored in a single run. 
For instance, SVDD~\citep{li2024derivative} proposes selecting the highest-reward particle at every denoising step, while CoDe~\citep{singh2025code} extends this idea by performing selection only at specific intervals, effectively balancing computational efficiency and sample diversity.
Generalizing the particle sampling paradigm, Sequential Monte Carlo (SMC) methods~\citep{wu2023practical,cardoso2023monte,kim2025test,dou2024diffusion} adopt a principled probabilistic framework that maintains a population of particles and iteratively performs importance weighting, resampling, and proposal optimization.
Beyond sampling-based strategies, recent work has investigated search-based decoding strategies, including tree search~\citep{li2025dynamic,zhang2025inference,ramesh2025test} and Monte Carlo Tree Search (MCTS)~\citep{xie2024monte,chen2024alphamath,zhang2024accessing,zhou2023language}, to improve alignment and reasoning performance by systematically exploring a larger set of candidate trajectories.
Although sampling-based approaches~\citep{singhal2025general,pani2025test,dang2025inference} have made early inroads into TTA for MDLMs, search-based methods remain underexplored. They have shown strong results in ARMs and continuous diffusion models. Extending them to MDLMs is still an open problem, which motivates this work.

\section{More Details on MDLMs}
\label{app:more_details_on_mdlms}

For completeness, we provide additional technical details underlying MDLMs used in the main text. Specifically, we (\romannumeral1) derive the form of the reverse process used during sampling, (\romannumeral2) present the training objective and its connection to the negative Evidence Lower Bound (negative ELBO), and (\romannumeral3) describe the deterministic constraints imposed on the denoiser outputs.

\paragraph{Derivation of the reverse process.}
The reverse process follows Bayes' rule
\begin{equation}
q(\bm{z}_s | \bm{z}_t,\bm{x}) \propto q(\bm{z}_t | \bm{z}_s)q(\bm{z}_s | \bm{x})
\end{equation}
for \(s<t\). By conditional independence across positions~\citep{sahoo2024simple,shi2024simplified}, it suffices to consider a single token. If \(\bm{z}_t \neq \bm{m}\), then by the absorbing property of the mask token \(\bm{m}\) in the forward process, the token cannot have been masked at time \(s\). Consequently, the reverse transition is deterministic and simply copies back:
\begin{equation}
q(\bm{z}_s | \bm{z}_t,\bm{x})=\delta_{\bm{z}_s,\bm{z}_t}.
\end{equation}
If \(\bm{z}_t=\bm{m}\), only \(\bm{z}_s\in\{\bm{m},\bm{x}\}\) have support. Using the forward transition \(q(\bm{z}_t=\bm{m}| \bm{z}_s=\bm{m})=1\) and \(q(\bm{z}_t=\bm{m}| \bm{z}_s=\bm{x})=1-\alpha_t/\alpha_s\), together with the prior \(q(\bm{z}_s=\bm{m}| \bm{x})=1-\alpha_s\) and \(q(\bm{z}_s=\bm{x}| \bm{x})=\alpha_s\), the unnormalized weights are
\(
w_{\bm{m}}=(1-\alpha_s)
\)
and
\(
w_{\bm{x}}=\alpha_s-\alpha_t
\),
with normalization \(Z=w_{\bm{m}}+w_{\bm{x}}=1-\alpha_t\). Hence
\begin{equation}
q(\bm{z}_s | \bm{z}_t=\bm{m},\bm{x})
=\mathrm{Cat}\Bigl(\bm{z}_s;\frac{(1-\alpha_s)\bm{m}+(\alpha_s-\alpha_t)\bm{x}}{1-\alpha_t}\Bigr),
\end{equation}
which is the expression used in~\cref{eq:time_reversal}. Intuitively, the mass not yet absorbed by time \(s\) splits between staying masked and reverting to the clean token \(\bm{x}\), in proportions \((1-\alpha_s)\) and \((\alpha_s-\alpha_t)\), respectively.

\paragraph{Training and connection to the negative ELBO.}
MDLMs are trained by minimizing the negative ELBO, which decomposes into a data term and a diffusion term that fits the learned reverse kernel \(p_\theta(\bm{z}_s | \bm{z}_t)\) to the true reverse dynamics \(q(\bm{z}_s | \bm{z}_t,\bm{x})\)~\citep{sahoo2024simple,shi2024simplified}. 
In the discrete-time setting with \(T\) steps, taking \(s(i)=(i-1)/T\) and \(t(i)=i/T\), the diffusion term simplifies to a time-weighted masked cross-entropy:
\begin{equation}
\mathcal{L}_{\mathrm{diff}}
=\sum_{i=1}^{T}\mathbb{E}_{\bm{x},\bm{z}_{t(i)}}\biggl[
\frac{\alpha_{t(i)}-\alpha_{s(i)}}{1-\alpha_{t(i)}}
\bigl(-\log \langle \bm{x}_\theta(\bm{z}_{t(i)},t(i)),\bm{x}\rangle\bigr)
\biggr].
\end{equation}
Taking \(T\to\infty\) yields the continuous-time negative ELBO (with \(\alpha_t'\) the derivative of \(\alpha_t\))
\begin{equation}
\mathcal{L}_\infty
=\int_{0}^{1}\frac{\alpha_t'}{1-\alpha_t}
\mathbb{E}_{\bm{x},\bm{z}_t\sim q(\cdot\mid \bm{x})}\bigl[
-\log \langle \bm{x}_\theta(\bm{z}_t,t),\bm{x}\rangle\bigr]\mathrm{d}t.
\end{equation}
For sequences with token-wise independent corruption and a factorized decoder (\cref{eq:naive_parallel_sampling}), this becomes a weighted average of MDLM losses over masked positions:
\begin{equation}
\mathcal{L}_\infty
=\int_{0}^{1}\frac{\alpha_t'}{1-\alpha_t}
\mathbb{E}\biggl[
\sum_{\ell=1}^{L}\bm{1}_{\{\bm{z}_t^{\ell}=\bm{m}\}}
\bigl(-\log \langle \bm{x}_\theta^{\ell}(\bm{z}_t^{1:L},t),\bm{x}^{\ell}\rangle\bigr)
\biggr]\mathrm{d}t.
\end{equation}
Thus the continuous-time training loss is exactly the negative ELBO specialized to masking, penalizing reconstruction errors only at positions that are masked by the forward process.

\paragraph{Denoiser output constraints.} We impose two output-time constraints on \(\bm{x}_\theta\) via deterministic post-processing (i.e., by substitution rather than learning):
(\romannumeral1) \emph{Zero masking probabilities}, by setting the mask logit to \(-\infty\), so \(\langle \bm{x}_\theta(\bm{z}_t,t),\bm{m}\rangle=0\);
(\romannumeral2) \emph{Carry-over unmasking}, i.e., for any position already unmasked, copy through the observed value so that if \(\bm{z}_t\neq \bm{m}\) then \(\bm{x}_\theta(\bm{z}_t,t)=\bm{z}_t\).
These constraints reflect the absorbing dynamics and ensure that only masked positions contribute to the loss, tightening the bound and stabilizing training and sampling.

\section{Theoretical Guarantee of \textsc{TReASURe}}
\label{app:theoretical_guarantee_Of_treasure} 

\UnmaskBranch*
\begin{proof}
Let \(n\) denote the number of masked positions at the parent node, and let \(\tau_n\) be the current time. Reparameterize time by \(\gamma(t) = -\log(1-\alpha_t)\), and discretize \(\gamma\) with uniform step size \(h\). Suppose we run na\"ive parallel sampling and terminate as soon as any position unmasks.

Consider any masked position \(\ell\). For \(s'< s\) such that \(\gamma_{s'}=\gamma_s-h\), the exact reverse transition gives
\begin{equation}
\mathbb{P}\bigl(\bm{z}_{s'}^{\ell} = \bm{m} \big| \bm{z}_{s}^{\ell} = \bm{m}, \bm{x}_\theta\bigr) = \frac{1 - \alpha_{s'}}{1 - \alpha_{s}} = \exp(-(\gamma_{s'}-\gamma_s)) = \exp(-h)
\end{equation}
This means that the probability of the token \(\bm{z}_s^\ell\) remaining masked is \(\exp(-h)\), and the probability it unmasks is \(1 - \exp(-h)\). Under na\"ive parallel sampling, the reverse transition is factorized as
\begin{equation}
p_\theta(\bm{z}_{s'}^{1:L}|\bm{z}_{s}^{1:L}) = \prod_{\ell=1}^{L} q(\bm{z}_{s'}^{\ell}| \bm{z}_{s}^{1:L}, \bm{x}_\theta),
\end{equation}
so the transitions of masked tokens are conditionally independent. Therefore, for the \(n\) masked tokens, the probability that no token unmasks in one step is
\begin{equation}
(\exp(-h))^n = \exp(-nh).
\end{equation}
Thus, the probability that at least one token unmasks in a step is \(p_h \coloneqq 1 - \exp(-nh)\).

Let \(G_h\) be the number of steps until the first unmasking occurs. Then for any \(k\geq 0\),
\begin{equation}
\mathbb{P}(G_h > k) = (\exp(-nh))^{k} = \exp(-nhk).
\end{equation}
Since the run halts at the first unmasking step, \(G_h\) is geometric with success probability \(p_h=1-\exp(-n h)\) on \(\{1,2,\dotsc\}\).
Hence
\begin{equation}
\mathbb{E}[G_h] = \frac{1}{p_h} = \frac{1}{1-\exp(-n h)}.
\end{equation}
Now suppose we wish to obtain \(b(n)\) child nodes by restarting from the same parent node and stopping each run at its first unmasking. Ignoring collisions (so that \(b(n)\) runs yield \(b(n)\) distinct children), the total number of model evaluations is \(\sum_{r=1}^{b(n)} G_h^{(r)}\) with \(G_h^{(r)}\) i.i.d.~as \(G_h\). Therefore,
\begin{equation}
\mathbb{E}[\mathrm{evals}]
= \sum_{r=1}^{b(n)} \mathbb{E}[G_h^{(r)}]
= \frac{b(n)}{1-\exp(-n h)}.
\end{equation}

When \(b(n)=1\), the \(\gamma\)-time waiting \(T_h \coloneqq h G_h\) satisfies
\begin{equation}
\mathbb{P}(T_h > s) = \bigl(\exp(-n h)\bigr)^{\lfloor s/h \rfloor} \to \exp(-n s) \quad \text{as } h \to 0,
\end{equation}
so \(T_h\) converges in distribution to \(\mathrm{Exponential}(n)\). By the inverse-CDF representation, if \(u \sim \mathcal{U}(0,1)\) then
\begin{equation}
\Delta_n \coloneqq -\frac{1}{n}\log u \sim \mathrm{Exponential}(n),
\end{equation}
and hence
\begin{equation}
\gamma(\tau_{n-1}) = \gamma(\tau_n) + \Delta_n.
\end{equation}
Mapping back via \(\gamma(t)=-\log(1-\alpha_t)\) yields the FHS draw
\begin{equation}
\tau_{n-1} = \alpha^{-1}\bigl(1 - u^{1/n}(1 - \alpha_{\tau_n})\bigr),\quad u \sim \mathcal{U}(0,1),
\end{equation}
and by exchangeability the committing index is uniform on the \(n\) masked positions. Thus for \(b(n)=1\) the na\"ive-parallel first-change law over \((\tau_{n-1}, \ell)\) coincides with FHS, while the expected number of model evaluations is \(1/\bigl(1-\exp(-n h)\bigr)\) for the grid scheme versus a single evaluation for \textsc{UnmaskBranch}.
\end{proof}

\ResubGap*
\begin{proof}
By~\Cref{assump:lipschitz}, for any realization \(\bm{X}^{1:L}\),
\begin{equation}
|r(\bm{X}^{1:L})-r(\hat{\bm{x}}_0^{1:L})|
 \le 
\beta d_{\mathrm{H}}(\bm{X}^{1:L},\hat{\bm{x}}_0^{1:L})
= \beta \sum_{\ell=1}^L \mathbf{1}_{\{X^\ell\neq \hat{\bm{x}}_0^\ell\}}.
\end{equation}
Taking expectations on both sides and applying the triangle inequality,
\begin{equation}
\Bigl|\mathbb{E}[r(\bm{X}^{1:L})]-r(\hat{\bm{x}}_0^{1:L})\Bigr|\leq
\beta \sum_{\ell=1}^L \mathbb{P}(X^\ell\neq \hat{\bm{x}}_0^\ell).
\end{equation}
For already unmasked positions \(\ell\notin\mathcal{I}_{n-1}\), we have \(X^\ell=\hat{\bm{x}}_0^\ell\), so these terms vanish. For \(\ell\in\mathcal{I}_{n-1}\), since \(X^\ell\sim\mathrm{Cat}(\cdot;\bm{\mu}_n^\ell)\) and \(\hat{\bm{x}}_0^\ell=\arg\max_v \bm{\mu}_n^\ell(v)\),
\begin{equation}
\mathbb{P}(X^\ell\neq \hat{\bm{x}}_0^\ell) = 1-\max_{v\in[V]} \bm{\mu}_n^\ell(v).
\end{equation}
Substituting gives the bound in the statement.
\end{proof}

\MonotoneM*
\begin{proof}
Write \(\mathcal{S}_n(m)\) for the set of parent nodes after the \(n\)th expansion step when using schedule \(m(\cdot)\). Thus at the start of the search we have \(\mathcal{S}_L(m) = \{(\bm z_L^{1:L}, \tau_L)\}\), corresponding to the all-mask initial state, while at the end we obtain \(\mathcal{S}_0(m)\), the collection of complete sequences with their rewards, from which the final \(\arg\max\) is taken. Define \(\mathcal{S}_n(m')\) analogously. Because the two runs are coupled (same \(\textsc{UnmaskBranch}\) uniforms, same committing indices, same model output \(\bm\mu_n\)), they produce the \emph{same} candidate pool \(\mathcal{C}_n\) at each level \(n\).

At level \(n\), selection applies \(\mathrm{TopK}_{m(n)}\) or \(\mathrm{TopK}_{m'(n)}\) to \(\mathcal{C}_n\) with the same deterministic tie-breaking. Since \(m'(n)\ge m(n)\), we have set inclusion
\begin{equation}
\mathrm{TopK}_{m'(n)}(\mathcal{C}_n) \supseteq\mathrm{TopK}_{m(n)}(\mathcal{C}_n),
\end{equation}
and hence \(\mathcal{S}_{n-1}(m')\supseteq \mathcal{S}_{n-1}(m)\). Inducting downwards from \(n=L\) to \(n=1\) yields \(\mathcal{S}_0(m')\supseteq \mathcal{S}_0(m)\).

The returned reward is the maximum of \(r(\cdot)\) over the respective final sets:
\begin{equation}
r_\star(m)=\max_{(\bm z_0^{1:L},\tau_0,r)\in \mathcal{S}_0(m)} r,\quad
r_\star(m')=\max_{(\bm z_0^{1:L},\tau_0,r)\in \mathcal{S}_0(m')} r.
\end{equation}
Since a maximum over a superset cannot be smaller, \(r_\star(m')\ge r_\star(m)\). The strictness condition follows immediately: if \(\mathcal{S}_0(m')\setminus \mathcal{S}_0(m)\) contains a sequence with reward exceeding \(\max_{\mathcal{S}_0(m)} r\), then the inequality is strict; otherwise the inequality reduces into an equality.
\end{proof}

\section{Experimental Details}
\label{app:experimental_details}

\paragraph{Details on beam width and tree width.} To ensure reproducibility, we report two structural statistics of the search process:

\begin{itemize}[leftmargin=*]
\item Beam width \(b(n)\): At the 
\(n\)th step, the beam width denotes the number of child nodes expanded from a single parent node. This reflects the local branching factor of the search tree and controls how many alternatives are explored before pruning.
\item Tree width \(m(n)\). Defined as the number of nodes retained after pruning at \(n\)th expansion step. In other words, the tree width corresponds to the effective number of candidates that survive pruning, thereby characterizing the degree of parallel exploration. This quantity directly determines the number of function evaluations (NFEs) required at each step, since every surviving node must be expanded.
\end{itemize}

In all of our experiments, both the beam width and the tree width are set as fixed constants for each configuration. We denote each configuration in the format ``\((b(n), m(n))\)''. For example, the settings \((5,2)\), \((6,4)\), \((8,6)\), \((12,8)\), \((20,16)\) indicate that each parent node expands into \(5\), \(6\), \(8\), \(12\), or \(20\) branches, while the tree width is fixed to \(2\), \(4\), \(6\), \(8\), or \(16\), respectively.

\paragraph{Details on baselines and evaluation metrics.} All baseline results reported in this paper are obtained directly from the official implementation of FK-Diffusion-Steering~\citep{singhal2025general}. All comparisons are based on running the released scripts without modification, unless otherwise specified. Similarly, all evaluation metrics are computed following the implementations provided in the same repository. We adopt the evaluation scripts released therein to guarantee fairness and consistency across methods. Specifically, the following pretrained models are used to calculate each metric:

\begin{itemize}[leftmargin=*]
\item \textbf{Perplexity} (\textsc{Gen}.~\textsc{PPL}): To encourage fluency during TTA, both the reward and the evaluation are based on generative perplexity computed with the pretrained GPT2-XL model~\citep{radford2019language}.
\item \textbf{Linguistic Acceptability} (\textsc{CoLA}): This metric favors grammatically well-formed sentences by employing a classifier~\citep{morris2020textattack} trained on the Corpus of Linguistic Acceptability (CoLA)~\citep{warstadt2019neural}. Importantly, CoLA classification accuracy is used consistently both as the reward signal during TTA and as the evaluation metric for reporting performance.
\item \textbf{Toxicity} (\textsc{Toxicity}): This metric leverages a toxicity detection classifier~\citep{morris2020textattack} to guide generation, with the task framed as red-teaming for harmful content. The goal is to assess model vulnerabilities in producing toxic or offensive outputs. Consistently, toxicity classification accuracy (with ``toxic'' as the positive class) is used both as the reward signal during test-time alignment and as the evaluation metric for reporting performance.
\item \textbf{Sentiment} (\textsc{Sentiment}): This metric employs a sentiment classifier~\citep{barbieri2020tweeteval} trained on social media text to guide outputs toward a desired polarity (e.g., positive sentiment). Consistently, sentiment classification accuracy (with ``positive'' as the target class) is used both as the reward signal during test-time alignment and as the evaluation metric for reporting performance.
\end{itemize}

\section{More Experiments}
\label{app:more_experiments}

\subsection{Variation of Rewards}
\label{app:variation_of_rewards}

\paragraph{Experimental settings.}  
To analyze the stability of reward signals during denoising, we generate full trajectories using first-hitting sampling (FHS). For each intermediate state \(\bm{z}_n^{1:L}\), we follow the reward estimation protocol of~\citet{singhal2025general}: we sample \(1{,}000\) candidate completions independently from the conditional distribution \(p_\theta(\cdot | \bm{z}_n^{1:L}, \tau_{n-1})\), compute the reward for each, and aggregate them to form an empirical distribution of reward estimates at that step. This allows us to track not only the mean reward but also its variability across the denoising process.

\paragraph{Results.}  
\Cref{fig:variation_of_reward_app} visualizes the distribution of estimated rewards at different denoising steps for two representative tasks, complementing~\cref{fig:prune_motivation}. We observe that the reward estimates fluctuate substantially, with wide interquartile ranges and large min--max spans. This high variance highlights a key weakness of direct sampling-based estimation: small stochastic differences in sampled tokens can lead to disproportionately large changes in the measured reward. As a result, pruning decisions based on these noisy estimates can be unreliable, especially in the middle stages of denoising where uncertainty is greatest. These findings underscore the motivation for our proposed \textsc{ResubstituteScore}, which provides low-variance, deterministic estimates and stabilizes pruning in tree search.

\begin{figure}[htb]
\centering
\includegraphics[width=0.48\textwidth]{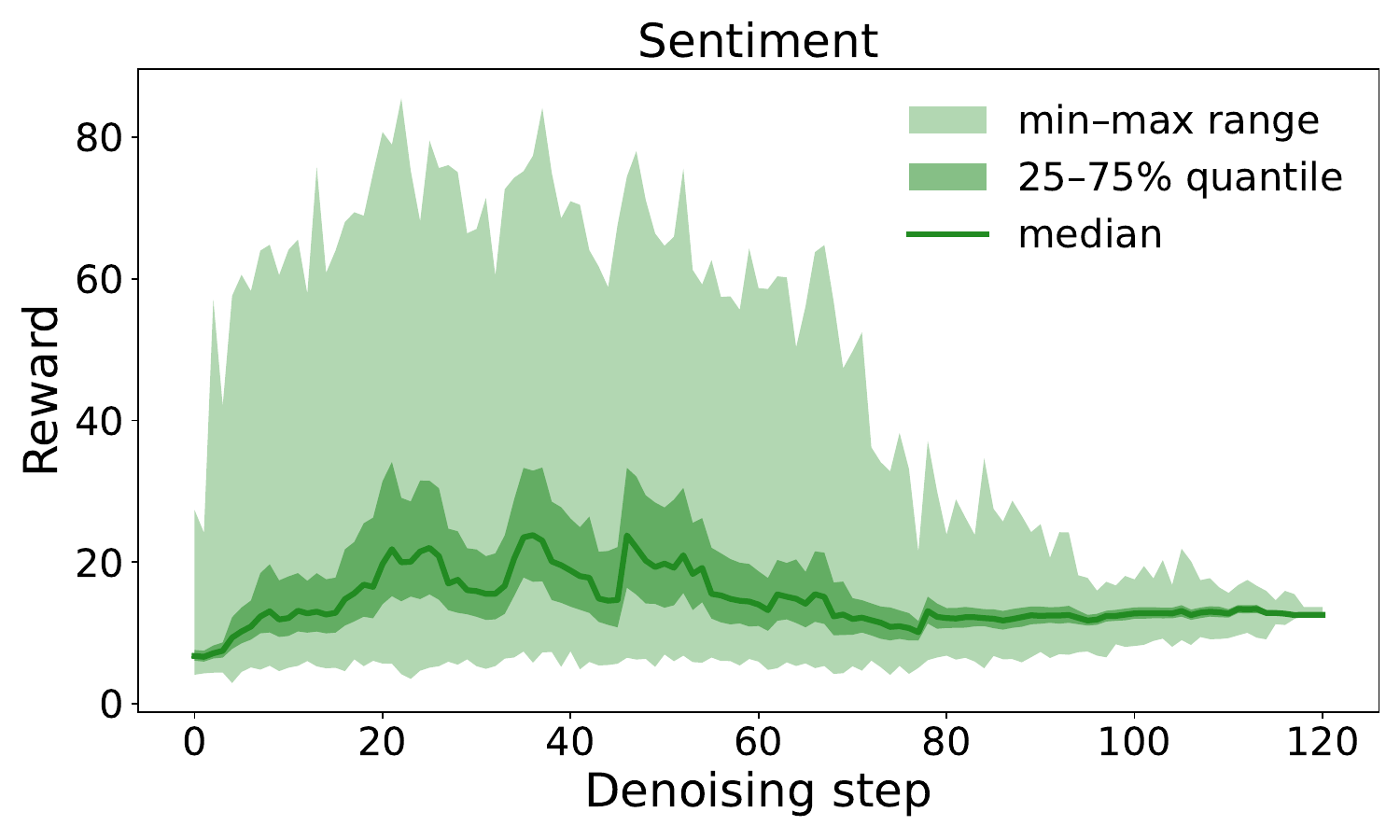}
\includegraphics[width=0.48\textwidth]{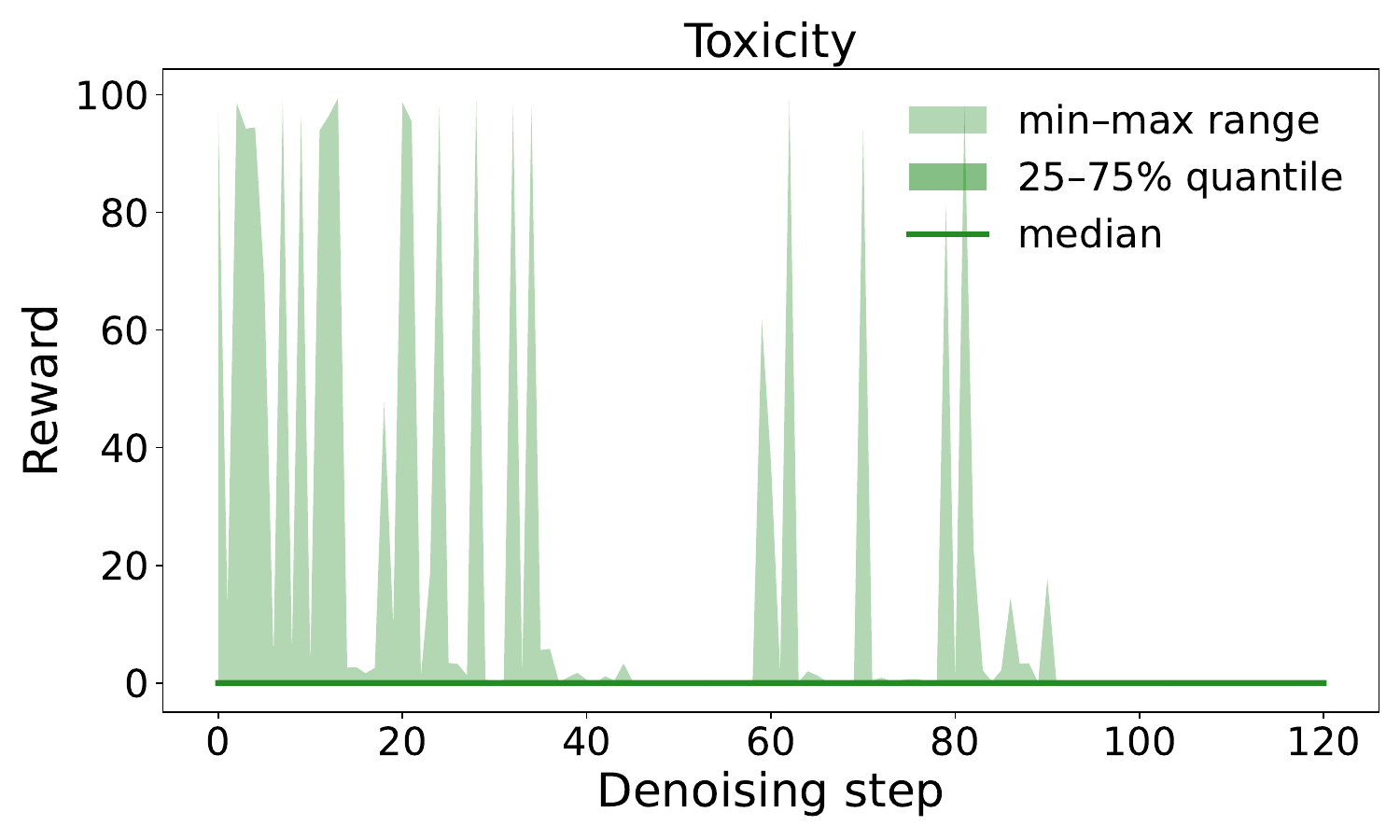}
\caption{\textbf{Variation of direct reward estimates across denoising steps.} 
We estimate rewards by sampling \(1{,}000\) completions per step without applying TTA, following~\citet{singhal2025general}. Both \textsc{Sentiment} (left) and \textsc{Toxicity} (right) exhibit high variance, with wide quantile bands and occasional outliers. In particular, \textsc{Toxicity} shows large fluctuations: although unconditional generations are usually non-toxic, rare toxic samples in the estimation process can cause sharp spikes. These results illustrate the instability of direct sampling-based scoring and motivate the need for low-variance alternatives such as \textsc{ResubstituteScore}.}
\label{fig:variation_of_reward_app}
\end{figure}

\subsection{Effect of Increasing Unmasking Groups}

\paragraph{Experimental settings.} In addition to beam width \(b(n)\) and tree width \(m(n)\), we also examine the effect of increasing the number of \emph{unmasking groups} in FHS. In the standard setting, FHS unmasks a single token uniformly from the remaining masked positions. Here, we generalize this by unmasking \(k\) tokens simultaneously. For each of the \(k\) positions, we compute the marginal distribution and select the top-\(b(n)\) candidates, yielding \(k \cdot b(n)\) child nodes. We then apply \textsc{ResubstituteScore} to estimate their rewards and retain the top \(m(n)\) nodes for expansion. We fix \(b(n)=5\) and \(m(n)=2\) across experiments, and vary \(k \in \{1,2,3\}\) (with \(k=1\) corresponding to the standard \textsc{TReASURe} algorithm).

\paragraph{Results.} \Cref{tab:unmasking_groups} shows that increasing the number of unmasking groups consistently improves performance across all metrics. Larger groups expand the search space more aggressively, enabling FHS to find higher-reward continuations while keeping the pruning capacity \(m(n)\) fixed. It is worth noting, however, that this strategy is a heuristic: unlike standard FHS, it no longer preserves the equivalence to naive parallel sampling. Nevertheless, the empirical gains suggest that parallel unmasking is a promising direction, and its theoretical implications merit further study.

\begin{table}[ht]
\centering
\renewcommand{\arraystretch}{1}
\setlength{\tabcolsep}{8pt}
\begin{tabular}{ccccc}
\toprule
\#\textsc{Unmasking Groups} & \textsc{CoLA} $\uparrow$ & \textsc{Toxicity} $\uparrow$ & \textsc{Sentiment} $\uparrow$ & \textsc{Gen}.~\textsc{PPL} $\downarrow$ \\
\midrule
\rowcolor{gray!15}
\(1\) & \(77.67\) & \(64.00\) & \(98.67\) & \(15.37\) \\
\(2\) & \(80.00\) & \(81.05\) & \(100.00\) & \(13.60\) \\
\(3\) & \(81.05\) & \(87.37\) & \(100.00\) & \(12.56\) \\
\bottomrule
\end{tabular}
\caption{\textbf{Effect of unmasking groups in FHS.} We vary the number of unmasking groups (\(k\)), which expands the child nodes per step from \(b(n)\) to \(k \cdot b(n)\). 
The gray row (\(k=1\)) is directly copied from \Cref{tab:main} as the baseline \textsc{TReASURe}. Increasing \(k\) yields consistently higher rewards and lower perplexity, indicating that parallel unmasking can improve search quality without additional per-step NFEs. Although this heuristic breaks the formal equivalence between FHS and naive parallel sampling, it nonetheless provides a simple and effective enhancement.}
\label{tab:unmasking_groups}
\end{table}

\subsection{Diversity Measure}

We further report the diversity of generated samples under different reward functions, as an extension of the main results in \Cref{tab:main}. 
Following standard practice, we adopt Distinct-\(n\) (\textsc{Dist-\(n\)}) metrics, including \textsc{Dist-1}, \textsc{Dist-2}, and \textsc{Dist-3}, which compute the ratio of unique \(n\)-grams to the total number of generated \(n\)-grams. 
These metrics capture the lexical diversity of the outputs: higher values indicate a greater variety of tokens and reduced repetition. 
All experimental settings remain identical to those described in the main text. 
As shown in~\Cref{tab:div_exp}, our method achieves the best overall performance while incurring only a marginal decrease in diversity.

\begin{table}[t]
\centering
\small
\renewcommand{\arraystretch}{0.9} 
\setlength{\tabcolsep}{8pt}       
\begin{tabular}{lcccc}
\toprule
\textsc{Method (Reward)} & \textsc{NFE} & \textsc{Dist-\(1\)} \(\uparrow\) & \textsc{Dist-\(2\)} \(\uparrow\) & \textsc{Dist-\(3\)} \(\uparrow\) \\
\midrule
MDLMs & \(1\) & \(63.60\) & \(92.63\) & \(94.00\) \\
\midrule
BoN (\textsc{CoLA}) & \(2\) & \(57.17\) & \(90.46\) & \(93.53\) \\
BoN (\textsc{TOXICITY}) & \(2\) & \(57.50\) & \(90.53\) & \(93.51\) \\
BoN (\textsc{SENTIMENT}) & \(2\) & \(56.99\) & \(90.75\) & \(93.85\) \\
BoN (\textsc{GEN. PPL}) & \(2\) & \(55.95\) & \(89.62\) & \(93.46\) \\
BoN (\textsc{CoLA}) & \(4\) & \(63.57\) & \(92.04\) & \(93.62\) \\
BoN (\textsc{TOXICITY}) & \(4\) & \(62.90\) & \(92.10\) & \(93.77\) \\
BoN (\textsc{SENTIMENT}) & \(4\) & \(62.50\) & \(92.58\) & \(94.11\) \\
BoN (\textsc{GEN. PPL}) & \(4\) & \(60.09\) & \(90.24\) & \(93.18\) \\
BoN (\textsc{CoLA}) & \(6\) & \(59.15\) & \(91.23\) & \(93.72\) \\
BoN (\textsc{TOXICITY}) & \(6\) & \(57.64\) & \(90.45\) & \(93.24\) \\
BoN (\textsc{SENTIMENT}) & \(6\) & \(57.09\) & \(91.50\) & \(94.20\) \\
BoN (\textsc{GEN. PPL}) & \(6\) & \(54.83\) & \(88.29\) & \(92.26\) \\
BoN (\textsc{CoLA}) & \(8\) & \(63.34\) & \(92.26\) & \(93.75\) \\
BoN (\textsc{TOXICITY}) & \(8\) & \(83.41\) & \(94.53\) & \(93.69\) \\
BoN (\textsc{SENTIMENT}) & \(8\) & \(62.56\) & \(92.57\) & \(94.28\) \\
BoN (\textsc{GEN. PPL}) & \(8\) & \(40.44\) & \(59.59\) & \(89.77\) \\
BoN (\textsc{CoLA}) & \(16\) & \(63.03\) & \(92.52\) & \(93.95\) \\
BoN (\textsc{TOXICITY}) & \(16\) & \(83.56\) & \(94.43\) & \(93.57\) \\
BoN (\textsc{SENTIMENT}) & \(16\) & \(62.27\) & \(92.57\) & \(94.29\) \\
BoN (\textsc{GEN. PPL}) & \(16\) & \(42.44\) & \(59.28\) & \(88.49\) \\
\midrule
FK-steering (\textsc{CoLA}) & \(2\) & \(57.73\) & \(90.75\) & \(93.73\) \\
FK-steering (\textsc{TOXICITY}) & \(2\) & \(57.03\) & \(89.89\) & \(93.29\) \\
FK-steering (\textsc{SENTIMENT}) & \(2\) & \(56.60\) & \(90.30\) & \(93.69\) \\
FK-steering (\textsc{GEN. PPL}) & \(2\) & \(55.82\) & \(89.63\) & \(93.58\) \\
FK-steering (\textsc{CoLA}) & \(4\) & \(63.65\) & \(92.11\) & \(93.86\) \\
FK-steering (\textsc{TOXICITY}) & \(4\) & \(62.46\) & \(91.46\) & \(93.37\) \\
FK-steering (\textsc{SENTIMENT}) & \(4\) & \(62.80\) & \(93.28\) & \(94.55\) \\
FK-steering (\textsc{GEN. PPL}) & \(4\) & \(59.99\) & \(90.57\) & \(93.41\) \\
FK-steering (\textsc{CoLA}) & \(6\) & \(58.87\) & \(90.98\) & \(93.54\) \\
FK-steering (\textsc{TOXICITY}) & \(6\) & \(56.82\) & \(89.78\) & \(92.89\) \\
FK-steering (\textsc{SENTIMENT}) & \(6\) & \(55.89\) & \(90.67\) & \(93.71\) \\
FK-steering (\textsc{GEN. PPL}) & \(6\) & \(55.27\) & \(88.22\) & \(92.47\) \\
FK-steering (\textsc{CoLA}) & \(8\) & \(64.06\) & \(92.50\) & \(93.96\) \\
FK-steering (\textsc{TOXICITY}) & \(8\) & \(63.09\) & \(91.86\) & \(93.20\) \\
FK-steering (\textsc{SENTIMENT}) & \(8\) & \(63.19\) & \(93.21\) & \(94.36\) \\
FK-steering (\textsc{GEN. PPL}) & \(8\) & \(59.40\) & \(89.30\) & \(92.34\) \\
FK-steering (\textsc{CoLA}) & \(16\) & \(64.11\) & \(92.56\) & \(93.91\) \\
FK-steering (\textsc{TOXICITY}) & \(16\) & \(64.21\) & \(92.59\) & \(93.43\) \\
FK-steering (\textsc{SENTIMENT}) & \(16\) & \(62.53\) & \(93.07\) & \(94.32\) \\
FK-steering (\textsc{GEN. PPL}) & \(16\) & \(58.29\) & \(87.93\) & \(91.54\) \\
\midrule
\textsc{TReASURe} (\textsc{CoLA}) & \(2\) & \(69.12\) & \(94.40\) & \(93.68\) \\
\textsc{TReASURe} (\textsc{TOXICITY}) & \(2\) & \(63.80\) & \(94.49\) & \(94.23\) \\
\textsc{TReASURe} (\textsc{SENTIMENT}) & \(2\) & \(61.05\) & \(94.91\) & \(94.98\) \\
\textsc{TReASURe} (\textsc{GEN. PPL}) & \(2\) & \(43.41\) & \(78.00\) & \(88.76\) \\
\textsc{TReASURe} (\textsc{CoLA}) & \(4\) & \(70.41\) & \(92.11\) & \(91.82\) \\
\textsc{TReASURe} (\textsc{TOXICITY}) & \(4\) & \(64.59\) & \(93.67\) & \(93.17\) \\
\textsc{TReASURe} (\textsc{SENTIMENT}) & \(4\) & \(60.09\) & \(94.39\) & \(94.83\) \\
\textsc{TReASURe} (\textsc{GEN. PPL}) & \(4\) & \(55.95\) & \(89.62\) & \(93.46\) \\
\textsc{TReASURe} (\textsc{CoLA}) & \(6\) & \(57.73\) & \(90.75\) & \(93.73\) \\
\textsc{TReASURe} (\textsc{TOXICITY}) & \(6\) & \(57.50\) & \(90.53\) & \(93.51\) \\
\textsc{TReASURe} (\textsc{SENTIMENT}) & \(6\) & \(56.99\) & \(90.75\) & \(93.85\) \\
\textsc{TReASURe} (\textsc{GEN. PPL}) & \(6\) & \(57.64\) & \(90.45\) & \(93.24\) \\
\textsc{TReASURe} (\textsc{CoLA}) & \(8\) & \(78.05\) & \(92.13\) & \(89.69\) \\
\textsc{TReASURe} (\textsc{TOXICITY}) & \(8\) & \(78.72\) & \(94.55\) & \(91.89\) \\
\textsc{TReASURe} (\textsc{SENTIMENT}) & \(8\) & \(80.85\) & \(97.14\) & \(94.92\) \\
\textsc{TReASURe} (\textsc{GEN. PPL}) & \(8\) & \(46.67\) & \(90.45\) & \(93.24\) \\
\textsc{TReASURe} (\textsc{CoLA}) & \(16\) & \(89.06\) & \(93.43\) & \(88.07\) \\
\textsc{TReASURe} (\textsc{TOXICITY}) & \(16\) & \(74.54\) & \(95.76\) & \(94.59\) \\
\textsc{TReASURe} (\textsc{SENTIMENT}) & \(16\) & \(83.19\) & \(97.01\) & \(94.73\) \\
\textsc{TReASURe} (\textsc{GEN. PPL}) & \(16\) & \(32.95\) & \(52.72\) & \(61.77\) \\
\bottomrule
\end{tabular}
\caption{\textbf{Diversity comparison under different methods and reward functions.} Results are reported in terms of \textsc{Dist}-\(n\) metrics (\textsc{Dist}-1/2/3).}
\label{tab:div_exp}
\end{table}

\section{Usage of Large Language Models}

We used Large Language Models (LLMs) to polish the writing, generate the cartoon explorer in~\cref{fig:illustration}, and assist with routine plotting code.

\end{document}